\documentclass[twoside]{article}
\usepackage[accepted]{aistats2026}

\usepackage[utf8]{inputenc} 
\usepackage[T1]{fontenc}    
\usepackage{hyperref}       
\usepackage{url}            
\usepackage{booktabs}       
\usepackage{amsfonts}       
\usepackage{nicefrac}       
\usepackage{microtype}      
\usepackage{xcolor}         
\usepackage{amsmath}
\usepackage{amsfonts}
\usepackage{microtype}
\usepackage{graphicx}
\usepackage{booktabs}
\usepackage{tabularx}
\usepackage{amsthm}
\usepackage{todonotes}
\usepackage{tikz}
\usepackage{fontawesome5}

\usepackage{natbib}
\usepackage{algorithm}
\usepackage{algpseudocode}
\usepackage{algorithm}
\usepackage{arydshln}
\usepackage{multirow}
\usepackage{shortcuts}
\usepackage{verbatim}
\usepackage{caption}
\usepackage{thmtools}
\usepackage{thm-restate}

\newtheorem{assump}{Assumptions}
\newtheorem{definition}{Definition}
\newtheorem{lemma}{Lemma}
\newtheorem{theorem}{Theorem}

\definecolor{afiablue}{RGB}{61,159,207}
\definecolor{afiared}{RGB}{167,75,68}
\definecolor{afialightblue}{RGB}{158,193,232}

\newcommand{\algoname}{EAGLE}

\title{Loss Gap Parity for Fairness\\ in Heterogeneous Federated Learning}

\begin{document}

\twocolumn[

\aistatstitle{Loss Gap Parity for Fairness in Heterogeneous Federated Learning}

\aistatsauthor{ Brahim Erraji \And Michaël Perrot \And  Aurélien Bellet }

\aistatsaddress{ Univ. Lille, Inria, CNRS,\\ Centrale Lille,  UMR 9189 -  \\CRIStAL, F-59000 Lille, France \And Univ. Lille, Inria, CNRS,\\ Centrale Lille,  UMR 9189 - \\ CRIStALF-59000 Lille, France  \And PreMeDICaL team, \\Inria, Idesp, Inserm,\\ Université de Montpellier} ]

\begin{abstract}
    While clients may join federated learning to improve performance on data they rarely observe locally, they often remain self-interested, expecting the global model to perform well on their own data. This motivates an objective that ensures all clients achieve a similar \emph{loss gap}—the difference in performance between the global model and the best model they could train using only their local data. To this end, we propose EAGLE, a novel federated learning algorithm that explicitly regularizes the global model to minimize disparities in loss gaps across clients. Our approach is particularly effective in heterogeneous settings, where the optimal local models of the clients may be misaligned. Unlike existing methods that encourage loss parity, potentially degrading performance for many clients, EAGLE targets fairness in relative improvements. We provide theoretical convergence guarantees for EAGLE under non-convex loss functions, and characterize how its iterates perform relative to the standard federated learning objective using a novel heterogeneity measure. Empirically, we demonstrate that EAGLE reduces the disparity in loss gaps among clients by prioritizing those furthest from their local optimal loss, while maintaining competitive utility in both convex and non-convex cases compared to strong baselines.
\end{abstract}

\section{INTRODUCTION}

Federated learning enables the training of machine learning models based on data stored across multiple clients without accessing or sharing it directly \citep{mcmahan2017communication}. It has emerged as a promising training framework for applications with regulated data access such as healthcare \citep{ long2020federated}, banking \citep{rieke2020future} and other sensitive applications \citep{li2020review}. Federated learning is typically orchestrated by a centralized, often trusted server responsible for aggregating local client models and synchronizing updates. This is the setting we consider in our work.

Federated optimization algorithms typically aim to minimize a (weighted) average of the local client training losses by generating a sequence of models based on information received from the clients, such as the gradient of the loss function on their respective data. Clients participate in federated learning by contributing their resources in exchange for access to knowledge aggregated from other clients, enabling the training of a model that generalizes well to the overall data distribution. However, they often also expect the model to perform well on their own local data distribution. For example, consider a medical diagnosis system in which each hospital seeks a model that generalizes to rare or unseen patient cases, while still maintaining high performance on the types of patients commonly encountered in its region. In this context, a model that exhibits uneven performance across datasets from different hospitals can be viewed as unfair. In fact, Federated Averaging (FedAvg)—introduced by \citet{mcmahan2017communication} and now the default algorithm in federated learning—often rewards clients that possess larger datasets, regardless of the relevance of their data to the overall utility. This happens because these clients are (i) typically weighted by their data size during aggregation; and (ii) allowed to run more local steps, since with a fixed batch size larger datasets require more updates per epoch \citep{wang2020tackling}, steering the global optimization toward their minimizers. These factors raise concerns about unfairness toward 
other clients, potentially discouraging their participation. As a result, promoting fairness has become a fundamental requirement in federated learning.

Similar to algorithmic fairness in centralized machine learning, determining what makes a federated model fair remains subject to debate and often depends on the specific application \citep{shi2023towards,benarba:hal-05093158,DBLP:journals/air/SalazarACA26}. A widely studied notion of fairness in federated learning is \textit{loss parity} \citep{lifair, li2021dittofairrobustfederated,Yue_2023}, which aims to equalize local losses between all participating clients. The underlying intuition is that similar loss values indicate the model has learned each client’s data distribution equally well. However, this implicitly assumes that losses are directly comparable across clients—a condition that often fails in heterogeneous settings. In practice, the data distributions of some clients may be intrinsically more challenging to model or suffer from lower quality or resolution, causing their losses to be on inherently different scales. Enforcing equalized losses in such cases can distort optimization and reduce the utility of the global model, sometimes causing a \textit{leveling down} effect—where fairness degrades outcomes for the best-off groups without benefiting the worst-off \citep{mittelstadt2023unfairness, zietlow2022leveling, maheshwari2023fair}. We illustrate this with a simple example involving two clients with different data distributions and a linear predictor in Figure~\ref{fig:exampel_synthetic_data}. Similar problems can arise with \textit{min-max} fairness \citep{mohri2019agnostic}, which optimizes for the worst-off client but may still compromise overall global performance.

\begin{figure}
    \centering
    \includegraphics[width=\columnwidth]{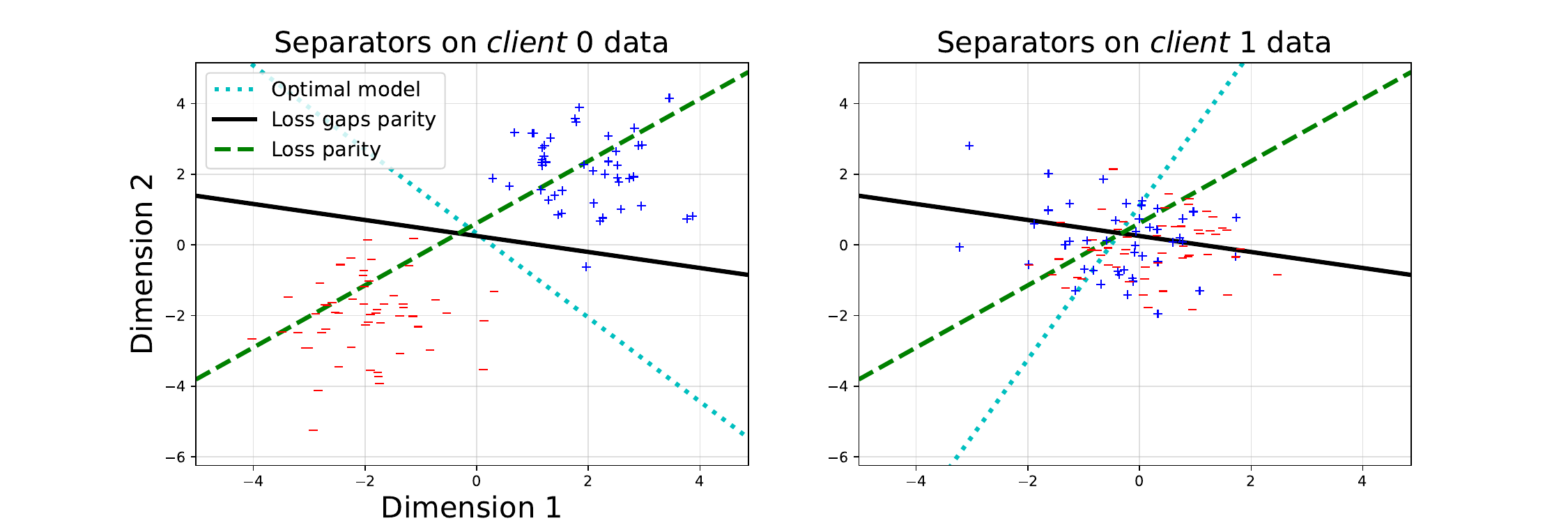} 
    \caption{In this example, we use the \textit{AFL} algorithm to enforce loss parity between the two clients, and our approach, \textit{\algoname{}}, to enforce loss gap parity. Enforcing loss parity favors \textit{$\clt~1$}, which has noisier data and a more complex prediction task, resulting in poorer performance for \textit{$\clt~0$}. In contrast, optimizing for loss gap parity handles this imbalance more effectively.}
    \label{fig:exampel_synthetic_data}
\end{figure}

To address these issues, we adopt a fairness objective centered on encouraging \emph{loss gap parity} across participating clients in federated learning. The loss gap of a client is defined as the difference between the loss of the global model and the lowest loss achievable using only the local data of that client. We consider a global model fair when these gaps are similar across clients, ensuring that all clients benefit relatively equally from collaboration.
To achieve this, we introduce \algoname{}, a novel algorithm that augments FedAvg with a regularization term explicitly penalizing the variance of client loss gaps. By tuning the regularization weight, \algoname{} allows one to trade off overall accuracy against fairness, consistently reducing disparities in relative improvements compared to optimal local model of each client.

Theoretically, we provide convergence guarantees for \algoname{} under non-convex local losses, showing that it reaches a saddle point of the regularized objective. We further connect these guarantees to the standard federated learning objective via a novel heterogeneity measure, highlighting how greater client heterogeneity can reduce global utility when enforcing fairness.
Empirically, we demonstrate the effectiveness of \algoname{} on synthetic data as well as the EMNIST and DirtyMNIST benchmarks, with both linear and convolutional models. Our results show that \algoname{} consistently reduces the variance in loss gaps across clients while maintaining comparable or improved overall performance relative to FedAvg and strong loss parity and min-max fairness baselines.

\paragraph{Related work.}
Several approaches in the literature define fairness in federated learning through client losses, either by reducing disparities in losses across clients or by optimizing for the worst-case loss \citep{lifair, li2021dittofairrobustfederated, mohri2019agnostic}. The motivation is that a fairer model should perform evenly across clients.  However, in heterogeneous settings where clients have different optimal local losses and misaligned optimal models, defining fairness solely based on absolute losses can undesirably penalize clients with easier data, effectively forcing them to perform as poorly as clients with noisier or more complex data.
To the best of our knowledge, FOCUS \citep{chu2023focusfairnessagentawarenessfederated} is the first work to explicitly question this prevailing notion of fairness and to propose \emph{loss gap parity} as a more suitable objective in heterogeneous settings. Their method builds personalized models for clusters of clients and evaluates fairness by measuring the loss gap of each client relative to its cluster model. However, their theoretical guarantees rely on very strong assumptions on the data distributions, limiting their applicability. In contrast, our approach is more general: it directly regularizes loss gap variance across all clients and provides explicit control over the fairness–utility trade-off via a tunable parameter $\lambda$.
In a parallel work, \citet{carey2025achieving} propose a unified framework for fairness in federated learning, introducing the notion of \emph{desert fairness}, which aims to equalize performance in proportion to the optimal local loss of each client. While their objective is related to ours, their approach again offers limited flexibility in balancing fairness and utility, and lacks formal theoretical guarantees.

\section{FAIR FEDERATED LEARNING: FROM LOSS PARITY TO LOSS GAP PARITY}

In this paper, we consider a cross-silo federated learning setup with $K$ clients collaborating on a supervised learning task. Let $\mathcal{X}$ denote the feature space and $\mathcal{Y}$ the label space.  Each client $k$ has access to a local dataset $D_k = {(x^i_k, y^i_k)}_{i=1}^{n_k}$ of $n_k$ examples in $\mathcal{X} \times \mathcal{Y}$. We consider models $\theta : \mathcal{X} \to \mathcal{Y}$ in a hypothesis space $\mathcal{H}$ equipped with a norm $\|\cdot\|$.

The general goal is to learn a global model that performs well across the data of all clients. Let $\ell(\theta, x, y)$ denote a loss function measuring the discrepancy between the prediction of a model $\theta(x)$ and the true label $y$. The standard federated learning objective is then
\begin{align*}
&L(\theta) := \frac1K \sum_{k=1}^{K} L_k(\theta),\\  &\text{ with } L_k(\theta) := \frac{1}{n_k} \sum_{i=1}^{n_k} \ell(\theta, x_k^i, y_k^i),
\end{align*}
which corresponds to minimizing the sum of local empirical losses $L_1,\dots,L_K$.\footnote{While arbitrary weights can be applied to the clients, we consider uniform weights here for simplicity.}

\looseness=-1 Optimizing the sum of local losses can lead to significant performance differences across clients when their data distributions are heterogeneous. 
Data heterogeneity occurs when clients’ data differ in their feature distributions, label distributions, or the conditional relationship between features and labels. This variability poses a fundamental challenge in federated learning \citep{kairouz2021advances}, calling into question whether collaboration truly benefits individual clients \citep{cho2022to}.

To address the uneven client performance caused by heterogeneity, a common strategy is to enforce some form of \emph{loss parity} \citep{lifair, li2021dittofairrobustfederated}, aiming for a model that achieves roughly equal local losses across all clients. Here, we consider a definition of approximate loss parity based on the variance of the local losses.
\begin{definition}[$\eps$-Loss Parity]\label{def:loss-parity}
    A model $\theta \in \mathcal{H}$ satisfies $\eps$-loss parity if the variance of the client losses $L_k(\theta)$ across the $K$ clients is at most $\eps$, that is:
    \begin{align}
    \label{eq:standard_obj}
        \mathbb{V}\left(\{ L_k(\theta)\}_{k = 1}^K\right) \leq \eps.
    \end{align}
    where $\mathbb{V}$ denotes the sample variance.
\end{definition}

By directly comparing local losses across clients, loss parity implicitly assumes that losses are on the same scale. In heterogeneous settings, however, this assumption often fails, as some clients may have noisier or intrinsically harder-to-predict data. In such cases, achieving loss parity may require to drastically reduce the utility of the global model on best-off clients. In particular, for any model $\theta$ that satisfies exact loss parity, i.e., $\mathbb{V}\left(\{ L_k(\theta)\}_{k = 1}^K\right) = 0$, we have: $\forall k \in [K]: L_k(\theta) \geq \max\limits_{k' \in [K]} L^*_{k'}$ where $L_k^{*} := \min\limits_{\theta \in \mathcal{H}} L_k(\theta)$. Consequently, clients with lower optimal losses must give up performance to achieve loss parity.

To address this limitation, we adopt in this work a more suitable notion of fairness that focuses on equalizing \emph{loss gaps} relative to the optimal local loss of each client.
\begin{definition}[$\eps$-Loss Gap Parity
]\label{def:loss-gaps-parity}
    A model $\theta \in \mathcal{H}$ satisfies $\eps$-loss gap parity if the variance of client loss gaps $r_k(\theta) := L_k(\theta) - L_k^*$ across the $K$ clients is at most $\eps$, that is:
    \begin{equation}
        \mathbb{V}\left(\{ r_k(\theta)\}_{k = 1}^K\right) \leq \eps.
    \end{equation}
\end{definition}
\looseness=-1 Unlike loss parity, loss gap parity promotes fairness in \emph{relative} performance. For instance, under exact loss gap parity (with $\eps=0$), each client achieves the same excess risk relative to its own optimal local loss. This avoids the leveling-down effect of loss parity: clients with noisier or harder-to-predict data naturally have higher optimal losses, so the global model is not forced to degrade the performance of easier clients to match them. On the other hand, when clients have identical or comparable optimal local losses, loss gap parity effectively reduces to loss parity, since the variance is unchanged by constant additive shifts. Moreover, in homogeneous data settings, the local losses of clients remain similar for any model, so training with objective~\eqref{def:loss-gaps-parity} already yields fairness under both definitions.

In the example of Figure~\ref{fig:exampel_synthetic_data}, heterogeneity in data distributions creates different levels of task complexity across clients, leading their local optimal models to attain different minimum losses. Enforcing loss parity in this setting compels the global model to keep losses high even for clients with simpler distributions, which can be considered suboptimal in terms of overall performance. By contrast, loss gap parity normalizes losses relative to the local optimum of each client, making them comparable and promoting more balanced—and less trivial—solutions.

\section{EAGLE: A FEDERATED LEARNING ALGORITHM FOR LOSS GAP PARITY}

In this section, we introduce \algoname{}, a new algorithm to enforce loss gap parity in federated learning. Throughout the presentation, we assume that the optimal local losses $L_k^*$ are known, and defer the discussion on how to approximate them in practice to Section~\ref{sec:optimallocallosses}.

To enforce loss gap parity when learning federated models, we propose to add a new regularization term to the standard federated learning objective to penalize models with high variance in loss gaps. Hence, we consider the following optimization problem:
\begin{align}
 &\arg\min\limits_{\theta \in \mathcal{H}} F(\theta) 
  :={} \frac{1}{K}\sum_{k=1}^K  L_{k}(\theta) + \mathbb{V}\left(\{ r_k(\theta)\}_{k = 1}^K\right)\label{eq:variance-obj}\\
&= \frac{1}{K}\sum_{k=1}^K \underbrace{\bigg[L_{k}(\theta) +  \frac{\lambda}{K-1} \sumkkp  (r_{k}(\theta) - r_{k'}(\theta))^2\bigg]}_{:= F_k(\theta)}, \label{eq:regularized-obj}
\end{align}%
where $r_{k}(\theta) := L_{k}(\theta) - L_{k}^* $ is the gap of client $k$ and the regularization parameter $\lambda > 0$ controls the degree to which loss gap parity should be respected. We note that  Problem~\eqref{eq:variance-obj}  and Problem~\eqref{eq:regularized-obj} are strictly equivalent \citep{zhang2012some}, but choosing the latter makes it easier to solve the problem in a federated way. Indeed, the gradient of $F$ can then be written as 
\begin{align*}
    &\nabla F(\theta) = \frac{1}{K}\sum_{k=1}^K\nabla F_k(\theta)=\\& \frac{1}{K}\sum_{k=1}^K \Big( \underbrace{1+ \frac{4\lambda}{K-1}  \sumkkp (r_{k}(\theta) - r_{k'}(\theta))}_{:= w_{k}(\theta)} \Big) \nabla L_{k}(\theta).
\end{align*} 
To solve \eqref{eq:regularized-obj}, we will adopt an optimization scheme similar to that of \textit{FedAvg} \citep{mcmahan2017communication}, where clients will perform multiple local steps without communicating with the server or the other clients, and then average the obtained local models. A key challenge in optimizing \eqref{eq:regularized-obj} with \textit{FedAvg} is that computing $\nabla F_{k}(\theta)$—and, in particular, the weights $w_k(\theta)$—requires information from all clients. This necessitates inter-client communication after every local step, which contradicts the primary goal of reducing communication. To overcome this challenge, we, instead, choose to fix the weights after each round of communication and only update the local gradient loss after each local step. In other words, we approximate the local gradients by $\nabla \tilde{F}_{k}(\theta, \thetapr) = w_k(\thetapr) \nabla L_{k}(\theta)$ where $\thetapr$ refers to the model received from the server at the second to last synchronization step. We summarize this approach, that we call \algoname{}, in Algorithm~\ref{alg:algoname}.

\begin{algorithm}[t]
\caption{ \algoname{} algorithm, in {\color{red}red} are the added  steps compared to FedAvg.}
\label{alg:algoname}
\begin{algorithmic}[1]
\Require
    \State $K$ : number of clients, $T$ : number of communication rounds, $I$ : number of local epochs, $\eta$ : learning rate, $\theta^0$ : initial model parameters, {\color{red}{$\lambda$ : regularization parameter of \algoname{}}},  {\color{red}{$\{L_k^*\}_{k = 1}^K$ : optimal local losses for each client}}, $\{D_k\}_{k=1}^K$ : local datasets for each client

\Procedure{ServerExecute}{}
    \State {\color{red}$\forall k \in [K], w^0_k \gets 1$}
    \For{each round $t = 0, 1, ..., T-1$}
        \For{each client $k \in [K]$ \textbf{in parallel}}
            \State $\theta_k^{t+1}, r^t_k \gets$ \Call{ClientUpdate}{$k, \theta^t, {\color{red}w^t_k}$}
        \EndFor
        \State Server calculates updated weights {\color{red}$ \forall k \in [K], w^{t+1}_k \gets 1 + \frac{4 \lambda}{K - 1} \sumkkp(r^t_k - r^t_{k'})$}
        \State $\theta^{t+1} \gets \frac1K \sum_{k \in [K]} \theta_k^{t+1}$ 
    \EndFor
    \State \Return $\theta_T$
\EndProcedure

\Procedure{ClientUpdate}{$k, \theta^{t}, {\color{red} w^t_k}$}
    \State {\color{red}Calculate the loss gap $r^t_k$ on the model $\theta^{t}$.}
    \State $\theta^{t+1,0}_k \gets \theta^{t}$
    \For{each local epoch $\tau = 0, 1, ..., I-1$}
        \State Generate an unbiased gradient $\nabla  L(\theta^{t+1,\tau}_{k}; \zeta_k)$ 
        
        \State $\theta^{t+1,\tau + 1}_{k} \gets \theta^{t+1,\tau}_{k} - \eta 
            {\color{red}w^t_k} \nabla  L(\theta^{t+1,\tau}_{k}; \zeta_k)$ 
    \EndFor
    \State \Return $\theta_{k}, r_k$
\EndProcedure
\end{algorithmic}
\end{algorithm}

\paragraph{Choosing the hyperparameter $\lambda$.} Setting $\lambda = 0$ recovers the standard federated learning objective, while letting $\lambda \to \infty$ forces the solution of \eqref{eq:regularized-obj} to satisfy loss gap parity exactly for all clients. Empirically, we observe that beyond a certain threshold for $\lambda$, efforts to reduce the variance of loss gaps no longer benefit the worst-performing clients in terms of loss gap parity. Instead, overall performance suffers in order to achieve lower variance, resulting in a “leveling down” effect. The parameter $\lambda$ thus allows practitioners to control the extent to which loss gap parity is enforced among clients, depending on the specific application.

\subsection{Theoretical Analysis}
\label{sec:theory}
In this section, we analyze the convergence of our approach from two different perspectives. First, we show that \algoname{} converges to a stationary point of \eqref{eq:regularized-obj}. It is worth noting that obtaining stronger convergence guarantees, for example in terms of closeness to the optimal model, might be challenging in our setting. Indeed, the structure of $F$ contains differences of functions $L_k$, $L_{k'}$  and thus, even if we assume convexity of the functions $\{L_k\}_{k=1}^K$, $F$ remains non-convex in general. Second, we show that \algoname{} returns models that have relatively low gradient norms with respect to the standard FedAvg objective. This result allows to quantify the decrease in utility incurred by our approach compared to FedAvg when the client losses satisfy the Polyak-Łojasiewicz (PL) condition.
We defer all proofs to Appendix~\ref{appendix:proofs}.

To derive our theoretical results, we need several regularity assumptions on the loss of each client, as summarized below.
\begin{assump}\label{assumptions}
We assume that the loss functions $\{L_k\}_{k = 1}^K$ have the following properties:
\begin{itemize}
\item $\forall k \in [K]: L_{k}$ is twice differentiable.
\item The function $L$ is bounded from below by $L^*$.
\item (Unbiased stochastic gradients) $\E_{\zeta_k}[\nabla L_k(\theta; \zeta_k)] = \nabla L_{k}(\theta)$.
\item (Bounded variance) $\exists \sigma \geq 0:\E_{\zeta_{k}}[\| \nabla L_{k}(\theta; \zeta_{k}) - \nabla L_{k}(\theta)\|^2] \leq \sigma^2$.
\item (Smoothness) $\forall k \in [K]: L_{k}$ is $\beta$-smooth: $\forall \theta, \thetapr \in \mathcal{H} : \| \nabla L_{k}(\theta) - \nabla L_{k}(\thetapr) \| \leq \beta \| \theta - \thetapr \|$. 
\item  (Bounded stochastic gradients) $\exists B > 0, \forall \theta \in \mathcal{H}, \forall k \in [K]: \|\nabla L_{k}(\theta; \zeta_{k})\| \leq B$.
\end{itemize}
\end{assump}
These assumptions are fairly common in the literature on the analysis of gradient descent-based algorithms \citep{wang2021field}. They are, for example, satisfied by the cross-entropy loss as long as the feature space $\mathcal{X}$ and hypothesis space $\mathcal{H}$ are bounded.

In our analysis, we propose a novel heterogeneity measure that depends on the loss gaps, namely \mbox{$\Gamma := \sup\limits_{\theta \in \mathcal{H}}\max\limits_{k, k'} |r_k(\theta) - r_{k'}(\theta)|$}. This quantity measures the level of heterogeneity of data across clients based on their optimal losses. Indeed, if the data is perfectly IID and each client has enough data points, then $\Gamma$ would be small. Otherwise it grows as the functions $\{r_k\}_{k = 1}^K$ differ from each other. 

Based on the assumptions listed in Assumption~\ref{assumptions} and our heterogeneity measure, we can now present the main convergence theorem of the \algoname{} algorithm.

\begin{theorem}[Convergence to a solution of \eqref{eq:regularized-obj}]
\label{theorem:conv_F}
Let $\theta_k^{(t, \tau)}$ refer to the model of client $k$ after $t$ communication rounds and $\tau$ local steps and let $\btheta^{(t, \tau)} := \frac{1}{K} \sumk \theta_k^{(t, \tau)}$. Let $T$ be the total number of communication rounds and $I$ be the number of local steps between each communication round. Under Assumption~\ref{assumptions}, for $\eta \leq \frac{1}{(1+4\lambda \Gamma)\beta + 8 \lambda B^2}$, the sequence of models $\{\btheta^{(t, \tau)}\}_{t \geq 0, \tau \geq 0}$ generated by Algorithm~\ref{alg:algoname} satisfies:
\begin{align}
&\frac{1}{T I} \sum_{t = 1}^T \sum_{\tau = 1}^I \E[\| \nabla F(\btheta^{(t,\tau)}) \|^2] \\&\leq 2\frac{1}{\eta T I} (F(\btheta^{(1,1)}) -  F^*) +  \eta^2 I^2 \xi_1 +\eta  \frac{\sigma^2}{K}  \xi_2 , \nonumber
\end{align} 
with $F^{*} := \arg\min\limits_{\theta \in \mathcal{H}} F(\theta),  \xi_1 :=  2 B^2 (1 + 4\lambda \Gamma)^2 (\beta^2 + 32 \lambda^2  (   \Gamma^2 \beta^2   +   B^4   ))$ and $\xi_2 := ((1 + 4\lambda \Gamma  )\beta + 8\lambda B^2)(1 + 4\lambda \Gamma)^2$.

If we further choose  $\eta =  \frac{1}{\sqrt{TI}}$ and $1 \leq I \leq \sqrt{T}$, the rate can be shown to be sublinear in $T$, that is:
\begin{align}
&\frac{1}{T I} \sum_{t = 1}^T \sum_{\tau = 1}^I \E[\| \nabla F(\btheta^{(t,\tau)}) \|^2]
 \\&\leq \underbrace{2\frac{1}{\sqrt{IT}} (F(\btheta^1) -  F^*) +   \frac{I}{T} \xi_1 +\frac{\xi_2}{\sqrt{IT}} \frac{\sigma^2}{K}.}_{\mathcal{O}(\frac{1}{\sqrt{T}})} \nonumber
\end{align} 
\end{theorem}

\newcounter{savedtheoremone}
\setcounter{savedtheoremone}{\value{theorem}}

Under standard assumptions on the loss function, we have shown that a carefully chosen step size allows us to recover the standard sublinear convergence rate of algorithms based on SGD in the general non-convex case \citep{yu2019parallel}. In this result, $I \xi_1$ represents the error related to running \algoname{} for multiple local steps before averaging. Naturally, the larger the number of local epochs $I$ and the heterogeneity measure $\Gamma$, the larger this error term. Similarly, $\xi_2$ is an upper bound on the the smoothness parameter of $F$. In fact, when $\lambda \to 0$, we have that $\xi_2 \to \beta$. This recovers the terms introduced in \citet{yu2019parallel} for FedAvg.    

The next theorem establishes convergence guarantees for \algoname{} on the standard federated objective $L(\theta) := \frac{1}{K} \sumk L_k(\theta)$, in the sense of convergence to a neighborhood of stationary points.

\begin{theorem}[Convergence to a neighborhood of the standard federated objective]
\label{theorem:conv_L}
Let $\theta_k^{(t, \tau)}$ refer to the model of client $k$ after $t$ communication rounds and $\tau$ local steps and let $\btheta^{(t, \tau)} := \frac{1}{K} \sumk \theta_k^{(t, \tau)}$. Let $T$ be the total number of communication rounds and $I$ be the number of local steps between each communication round. Under Assumption~\ref{assumptions}, for $\eta \leq \frac{1}{(1+4\lambda \Gamma)\beta + 8 \lambda B^2}$, the sequence of models $\{\btheta^{(t, \tau)}\}_{t \geq 0, \tau \geq 0}$ generated by Algorithm~\ref{alg:algoname} satisfies:
\begin{align}
\frac{1}{T I} \sum_{t = 1}^T \sum_{\tau = 1}^I &\E[\| \nabla L(\btheta^{(t,\tau)}) \|^2] \leq 2\frac{1}{\eta T I} (L(\btheta^{(1,1)}) -  L^*) \nonumber\\&+  \underbrace{32 \lambda^2 \Gamma^2 B^2}_{\substack{\emph{neighborhood of the solution} \\ \emph{relative to the heterogeneity }}} + \eta \frac{\sigma^2}{K} \xi_2\\& +  8 \beta^2  \eta^2 I^2 (1 + 4\lambda \Gamma)^2 B^2 \nonumber,
\end{align} 
with $\xi_2 := ((1 + 4\lambda \Gamma  )\beta + 8\lambda B^2)(1 + 4\lambda \Gamma)^2$.

If we further assume that $L(\theta)$ is $\mu$-Polyak-Łojasiewicz (PL), that is $\forall \theta \in \mathcal{H}: \frac12\| \nabla L(\theta) \| \geq \mu (L(\theta) - L^{*})$, then we have that:
\begin{align}
&\frac{1}{T I} \sum_{t = 1}^T \sum_{\tau = 1}^I \E[L(\btheta^{(t,\tau)}) - L^*] \leq 4\frac{\mu}{\eta T I} (L(\btheta^{(1,1)}) -  L^*) \nonumber\\&  + 2 \eta \frac{\sigma^2}{K} \xi_2 +  16 \beta^2  \eta^2 I^2 (1 + 4\lambda \Gamma)^2 B^2+ \underbrace{62 \mu \lambda^2 \Gamma^2 B^2}_{{\emph{Loss in utility}}}.
\end{align}

\end{theorem}

\newcounter{savedtheoremtwo}
\setcounter{savedtheoremtwo}{\value{theorem}}

Theorem~\ref{theorem:conv_L}  bounds the potential loss in utility induced by the fairness regularization, as the only term that does not vanish with an appropriate choice of $\eta$ is of order $\mathcal{O}(\lambda^2 \Gamma^2)$. As expected, lower heterogeneity $\Gamma$ or smaller values of $\lambda$ lead to a reduced loss in utility.

\subsection{Practical Considerations}
\label{sec:optimallocallosses}

In this section, we tackle two issues that may arise when using \algoname{} in practice, namely the fact that optimal local losses are not available and that the added regularization term may create some instability in the optimization process.

\paragraph{Optimal local losses.} So far, we assumed that the optimal losses $L_k^*$ were available. However, this is an unrealistic assumption in practice since even getting a reliable approximation of this quantity would require each client to have access to sufficiently many examples to learn a good model on their own data distribution. In the experiments of Section~\ref{sec:experiments}, we tackle this issue by  splitting the data of each client into a training and a validation set. We then learn the best possible model on the train set of each client and approximate $L_k^*$ on their validation set. We provide additional details on how to obtain the best possible approximation of $L_k^*$ in Appendix~\ref{subsec:l_k_s}.

\paragraph{Instability of \algoname{}.}
Our approach can be seen as a reweighting approach where each client gets a new weight $w_k^t$ at each communication round. These weights are adaptive and are based on the level of loss gap parity of the global model. This can be equivalently viewed as performing SGD locally with varying learning rates. Many model architectures are sensitive to this hyperparameter, and incorrect choices can lead to exploding gradients, even in simple models \citep{meng2024gradient}. To avoid this issue, we normalize the weights $\{w^t_k\}_{k = 1}^K$ so that they have an $L_2$ norm of $1$. This preserves the sign of the weights and their relative importance, while keeping their magnitude constant, preventing unintended variations in the effective learning rate.

\section{EXPERIMENTS}
\label{sec:experiments}
In this section, we empirically evaluate \algoname{} against \textit{FedAvg} and two algorithms that aim for loss parity. This effectively allows us to highlight the issues of fairness unaware training (\textit{FedAvg}) and that of the training that aims for a fairness definition (loss parity) that may not adequately handle data heterogeneity.

The section is structured as follows. In the first part, we showcase experiments on logistic regression. We begin with a synthetic dataset where having clearly separable data in the input space is actually a disadvantage for clients under a loss-parity–enforcing algorithm. We then demonstrate how \algoname{} overcomes this issue by assigning higher weights to clients with linearly separable data. Next, we show that such heterogeneity in optimal local losses also arises in a real dataset (EMNIST \citep{cohen2017emnist}). In the second part, we present experiments with a CNN model to highlight the limitations of loss-parity–enforcing approaches and unconstrained training on more complex models. These experiments are conducted on both the EMNIST and DirtyMNIST datasets \citep{mukhoti2023deep}.

\paragraph{Baselines.} Throughout this section, we evaluate \algoname{} on accuracy and loss gap fairness metrics. We compare it with: (i) \textit{FedAvg} for unconstrained training; (ii)\textit{q-FedAvg (q-FFL)} \citep{lifair}, which aims to balance client losses by raising each local loss to the power $q + 1$, with $q$ a tunable parameter, thereby giving higher weights to the clients with larger losses; (iii) \textit{Agnostic Federated Learning (AFL)}~\citep{mohri2019agnosticfederatedlearning}, which adopts an \textit{egalitarian} strategy aimed at minimizing the loss of the worst-performing client.

\paragraph{Metrics.} We evaluate the algorithms in terms of fairness, measured by the variance of loss gaps, and utility, quantified as the balanced accuracy across clients. We also report the performance of the best and worst clients (based on loss gaps) to determine whether fairness is achieved by improving the utility of the worst-performing clients or by simply leveling down the better-performing ones.

In all experiments, the loss gaps ${r_k(\theta)}_{k=1}^K$ are computed on a held-out validation set to prevent overfitting to the training data. Note that loss gaps can be negative, indicating that federated learning improves client generalization. In other words, the federated model can outperform local training on the validation set. When all loss gaps are negative, achieving loss gap parity ensures that every client benefits equally from federated learning, receiving the same improvement in generalization on their respective validation sets.

\begin{figure}[t]
    \centering
    \begin{minipage}{0.49\columnwidth}
        \centering
        \includegraphics[width=\textwidth]{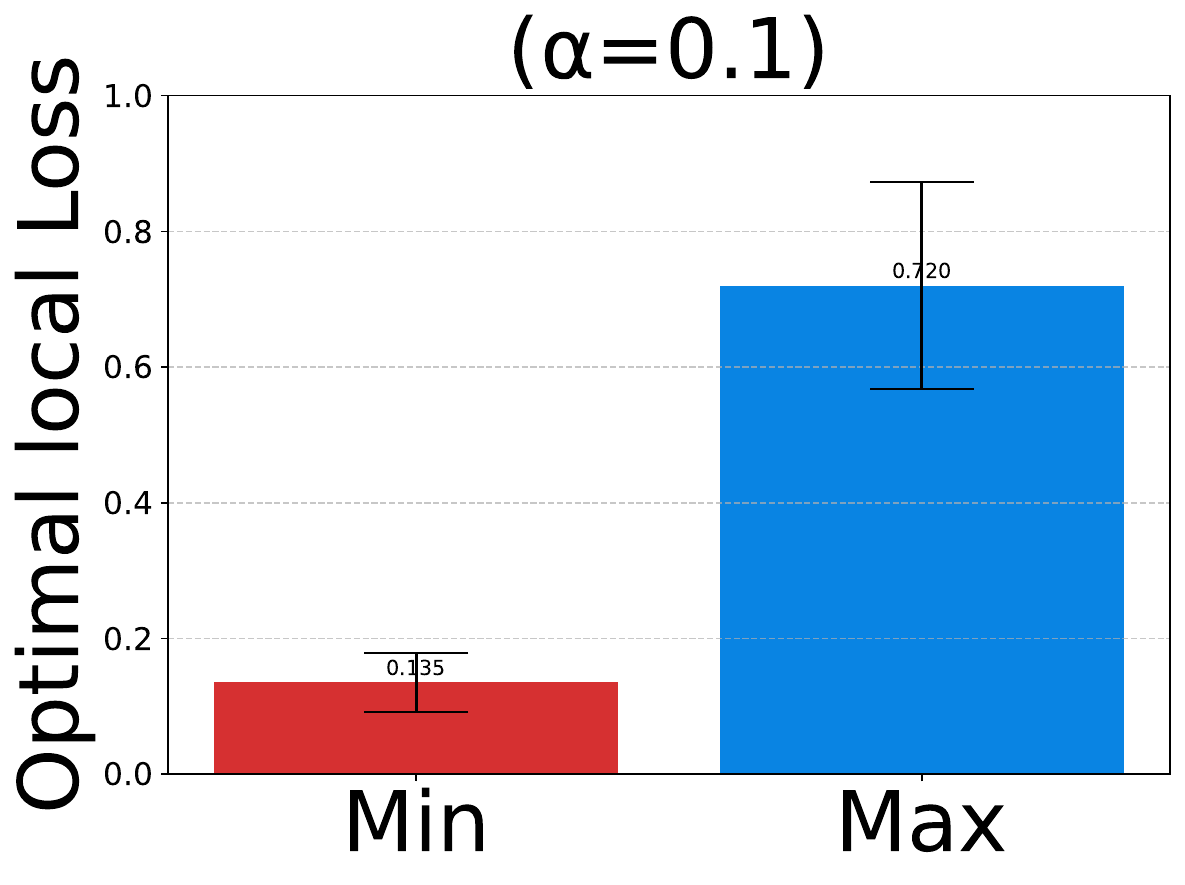}
        \caption*{$\alpha = 0.1$}
    \end{minipage}
    \begin{minipage}{0.49\columnwidth}
        \centering
        \includegraphics[width=\textwidth]{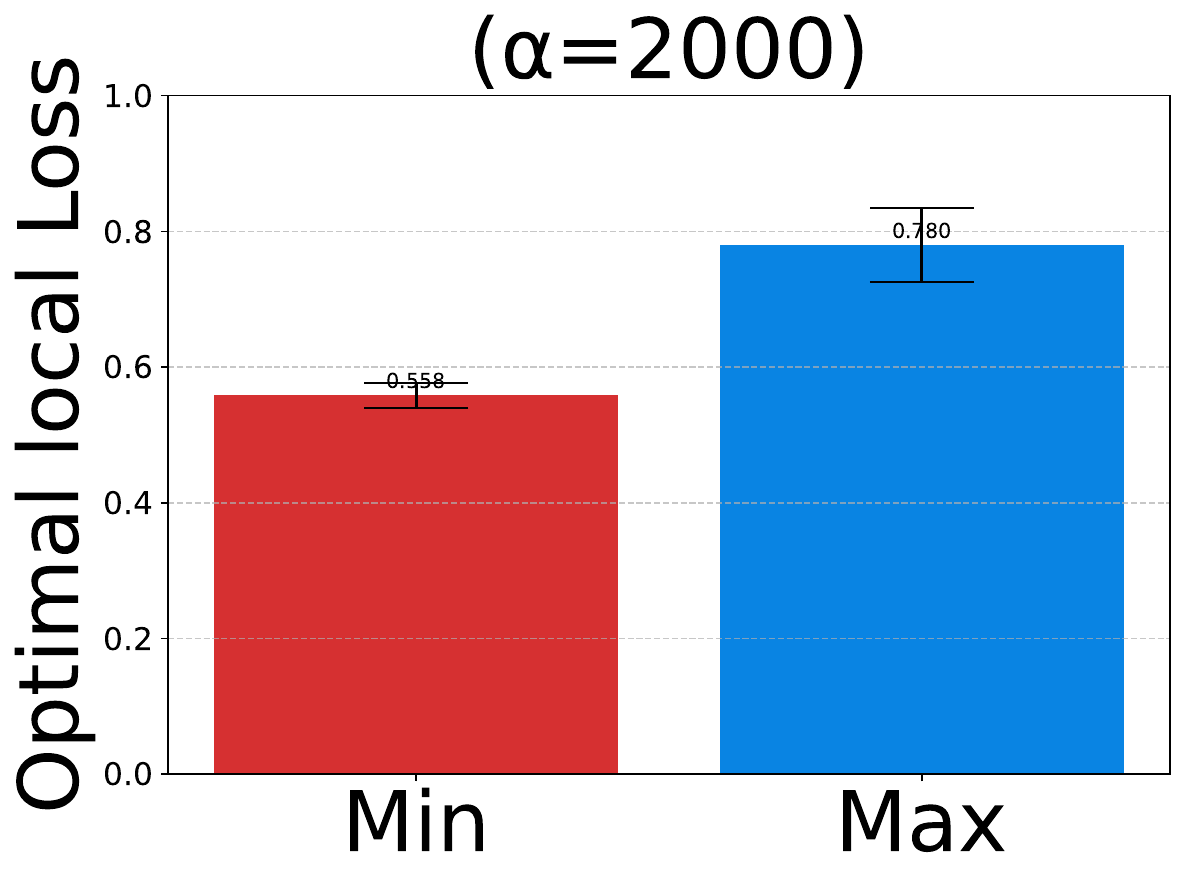}
        \caption*{$\alpha = 2000$}
    \end{minipage}
    \caption{EMNIST data: $\min \{L^*_k\}_{k \in [10]}$ (red) and $\max \{L^*_k\}_{k \in [10]}$ (blue) for a heterogeneous data split ($\alpha = 0.1$) and a homogeneous split ($\alpha = 2000$) using a linear model. As expected, the range of optimal local losses is larger for non-iid data.}
    \label{fig:dirichlet-opt-loss}
\end{figure}

\begin{figure}[t]
\centering 
\includegraphics[width=\columnwidth]{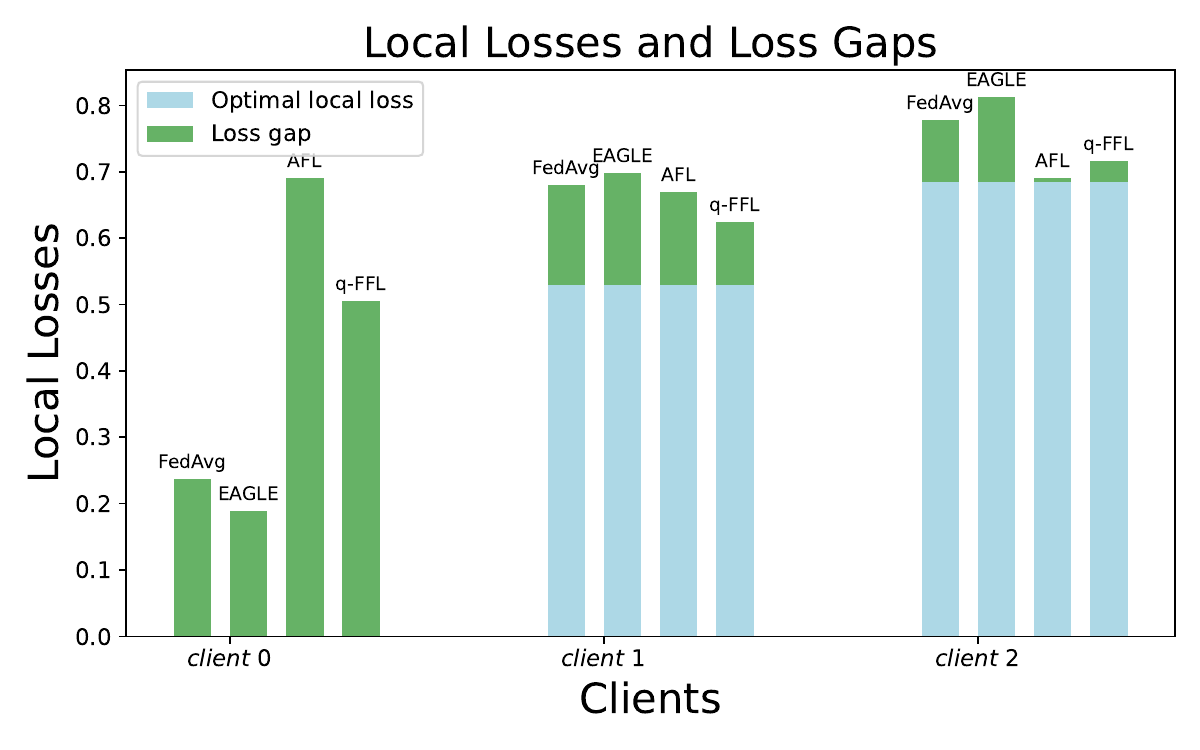}
\caption{Local losses and local gaps for \mbox{\algoname{}($\lambda = 2.0$), q-FFL ($q = 5.0$)} and AFL. The optimal local loss of $\clt~0$ is nearly $0$ by design and thus only its loss gap is visible in the plot. By balancing the losses across clients, q-FFL and AFL prioritize learning a good classifier for clients 1 and 2, which harms the performance of $\clt~0$. In contrast, by balancing the loss \emph{gaps} across clients,  \algoname{} allows all clients to benefit equally from federated learning.} \label{fig:obj-gaps-synthetic} 
\end{figure}

\paragraph{Datasets.} We experiment on the following datasets.
\begin{itemize}
\item \textbf{Synthetic data.} We consider three clients collaborating to fit a linear separator of $2$-dimensional data into two classes $Y = \{-1,1\}$, each client samples $100$
data points from the following underlying distributions: 
\begin{itemize}
\item $\clt~0$: \mbox{$X | Y = 1 \sim \mathcal{N}([2,2], \mathbb{I}_2)$} and $X | Y = -1 \sim \mathcal{N}([-2,-2], \mathbb{I}_2)$,
\item $\clt~1$: $X | Y = 1 \sim \mathcal{N}([0.5,0.5], \mathbb{I}_2)$ and \mbox{$X | Y = -1 \sim \mathcal{N}([-0.5,-0.5], \mathbb{I}_2)$},
\item $\clt~2$: $X | Y = 1 \sim \mathcal{N}([0.1,0.1], \mathbb{I}_2)$ and $X | Y = -1 \sim \mathcal{N}([-0.1,-0.1], \mathbb{I}_2)$,
\end{itemize}
where $\mathbb{I}_2$ is the $2\times 2$ identity matrix.  We also rotate the data of $\clt~2$ around $x_0 = \{0,0\}$ by $45$ degrees to avoid the case where all clients share an optimal separator. The resulting distributions are depicted in the Appendix (Figure~\ref{fig:rotated-data}).
\item \textbf{EMNIST.} To simulate federated learning with varying degrees of data heterogeneity, we split the EMNIST dataset \citep{cohen2017emnist} among 10 clients using a Dirichlet-based allocation method \citep{hsu2019measuring}. The heterogeneity of the split is parameterized by $\alpha$ (the smaller, the more heterogeneous). Details of the splitting procedure are provided in Appendix~\ref{appendix:EMNIST}. This heterogeneous allocation induces differences in the optimal local metrics across clients, as illustrated in Figure~\ref{fig:dirichlet-opt-loss}.
\item \textbf{DirtyMNIST.} The DirtyMNIST dataset labels examples as either clean or ambiguous, depending on how confidently humans can assign a label. We exploit this attribute to distribute 8,000 examples across 5 clients, each receiving a different proportion of ambiguous images—from 0\% in $\clt~0$ up to 100\% in $\clt~5$. The intuition is that a higher fraction of ambiguous examples leads to a higher optimal local loss for that client.
\end{itemize}

\subsection{Results}
Here, we illustrate the performance of the different algorithms on logistic regression with a linear model and on image classification using a CNN.

\begin{figure}[t]
    \centering
\hspace*{-.1cm}\includegraphics[width=1\columnwidth]{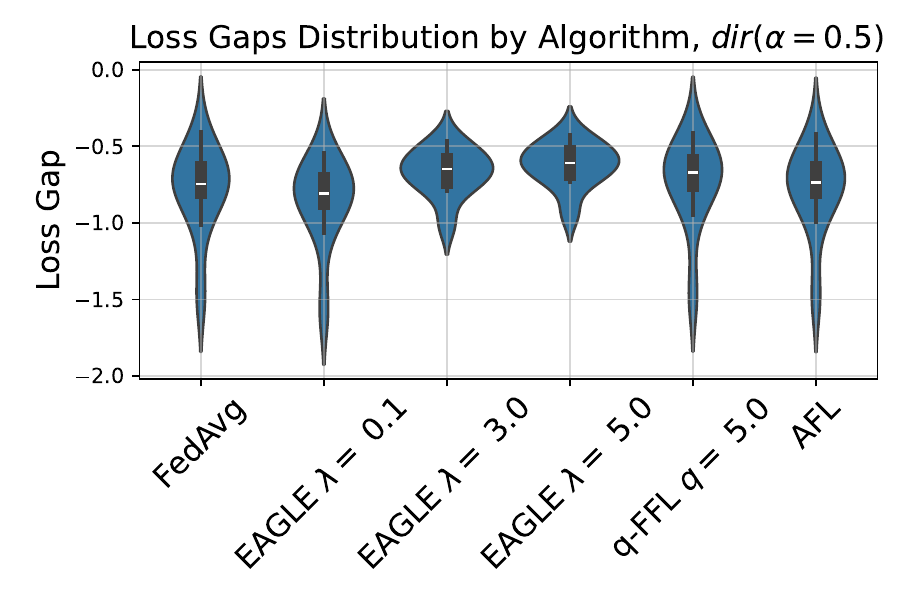}
    \caption{The distribution of loss gaps for the $10$ clients for a heterogeneous split ($\alpha=0.5)$ with CNN model. We notice that in the case of high non-iid data, \algoname{} reduces the gap of the worst performing client in terms of loss and achieves less variance then other baselines.}
\label{fig:cnn_alpha_05}
\end{figure}

\begin{table*}[t]
\caption{Performance of different algorithms on EMNIST with a highly heterogeneous split ($\alpha = 0.1$) and a linear model. \algoname{} achieves the best performance for the worst-performing client in terms of loss gaps, while maintaining accuracy comparable to the other baselines.}\label{table:linear_alpha_01}

\begin{tabular}{l|p{0.17\textwidth}|p{0.17\textwidth}|p{0.17\textwidth}|p{0.17\textwidth}}
\hline
Algorithm & \multicolumn{4}{c}{$\alpha = 0.1$} \\
 &  $ \max\limits_{k \in [K]} r_k(\theta) \downarrow$ &  $ \min\limits_{k \in [K]} r_k(\theta)\downarrow$ & accuracy $\uparrow$ &  $\mathbb{V}_{k \in [K]} r_k(\theta) \downarrow $ \\
\hline
AFL & 0.152 {\tiny($\pm$ 0.207)} & \textbf{-0.477 {\tiny($\pm$ 0.155)}} & 0.687 {\tiny($\pm$ 0.003)} & 0.037 {\tiny($\pm$ 0.026)} \\
FedAvg & 0.106 {\tiny($\pm$ 0.170)} & -0.459 {\tiny($\pm$ 0.211)} & \textbf{0.692 {\tiny($\pm$ 0.004)}} & 0.033 {\tiny($\pm$ 0.021)} \\
q-FFL $q=1.0$ & 0.104 {\tiny($\pm$ 0.178)} & -0.464 {\tiny($\pm$ 0.193)} & \textbf{0.692 {\tiny($\pm$ 0.002)}
} & 0.032 {\tiny($\pm$ 0.021)} \\
q-FFL $q=3.0$ & 0.109 {\tiny($\pm$ 0.186)} & -0.471 {\tiny($\pm$ 0.178)} & 0.691 {\tiny($\pm$ 0.004)} & 0.032 {\tiny($\pm$ 0.022)} \\
q-FFL $q=5.0$ & 0.118 {\tiny($\pm$ 0.189)} & -0.471 {\tiny($\pm$ 0.169)} & 0.689 {\tiny($\pm$ 0.004)} & 0.033 {\tiny($\pm$ 0.023)} \\
\hdashline
\algoname{} $\lambda=0.1$ & 0.096 {\tiny($\pm$ 0.155)} & -0.450 {\tiny($\pm$ 0.209)} & 0.691 {\tiny($\pm$ 0.007)} & 0.030 {\tiny($\pm$ 0.019)} \\
\algoname{} $\lambda=1.0$ & \textbf{0.073 {\tiny($\pm$ 0.123)}} & -0.381 {\tiny($\pm$ 0.199)} & 0.687 {\tiny($\pm$ 0.003)} & 0.020 {\tiny($\pm$ 0.013)} \\
\algoname{} $\lambda=2.0$ & 0.094 {\tiny($\pm$ 0.166)} & -0.355 {\tiny($\pm$ 0.268)} & 0.681 {\tiny($\pm$ 0.007)} & \textbf{0.019 {\tiny($\pm$ 0.010)}} \\
\algoname{} $\lambda=3.0$ & 0.107 {\tiny($\pm$ 0.171)} & -0.335 {\tiny($\pm$ 0.275)} & 0.677 {\tiny($\pm$ 0.007)} & 0.019 {\tiny($\pm$ 0.010)} \\
\algoname{} $\lambda=5.0$ & 0.203 {\tiny($\pm$ 0.181)} & -0.245 {\tiny($\pm$ 0.284)} & 0.650 {\tiny($\pm$ 0.007)} & 0.020 {\tiny($\pm$ 0.011)} \\
\hline
\end{tabular}
\end{table*}

\begin{table*}[t]
\caption{Performance of different algorithms on DirtyMNIST with a CNN model. \algoname{} achieves the best performance for the worst-performing client in terms of loss gaps, while slightly improving accuracy comparable to the other baselines. Note that all approaches generally achieve good accuracy because clients with ambiguous examples provide gradients that are also useful for the clean data.}  \label{table:main_dirty}
\begin{tabular}{l|p{0.17\textwidth}|p{0.17\textwidth}|p{0.17\textwidth}|p{0.17\textwidth}}
\hline
Algorithm 
 & $\max\limits_{k \in [K]} r_k(\theta) \downarrow$ & $\min\limits_{k \in [K]} r_k(\theta) \downarrow$ & accuracy $\uparrow$ & $\mathbb{V}_{k \in [K]} r_k(\theta) \downarrow$  \\
\hline
AFL & -0.016 {\tiny $\pm$ 0.013} & -0.822 {\tiny $\pm$ 0.017} & 0.882 {\tiny $\pm$ 0.003} & 0.083 {\tiny $\pm$ 0.000} \\
FedAvg & -0.083 {\tiny $\pm$ 0.004} & -0.877 {\tiny $\pm$ 0.046} & 0.889 {\tiny $\pm$ 0.007} & 0.081 {\tiny $\pm$ 0.010} \\
q-FFL $ q = 1.0 $ & -0.084 {\tiny $\pm$ 0.004} & -0.886 {\tiny $\pm$ 0.022} & 0.889 {\tiny $\pm$ 0.001} & 0.084 {\tiny $\pm$ 0.003} \\
q-FFL $ q = 3.0 $ & -0.062 {\tiny $\pm$ 0.013} & -0.883 {\tiny $\pm$ 0.020} & 0.883 {\tiny $\pm$ 0.005} & 0.088 {\tiny $\pm$ 0.006} \\
q-FFL $ q = 5.0 $ & -0.070 {\tiny $\pm$ 0.003} & -0.876 {\tiny $\pm$ 0.018} & 0.879 {\tiny $\pm$ 0.006} & 0.083 {\tiny $\pm$ 0.005} \\
\hdashline
EAGLE $ \lambda = 0.1 $ & -0.095 {\tiny $\pm$ 0.003} & \textbf{-0.944 {\tiny $\pm$ 0.018}} & \textbf{0.897 {\tiny $\pm$ 0.001}} & 0.093 {\tiny $\pm$ 0.005} \\
EAGLE $ \lambda = 0.3 $ & \textbf{-0.099 {\tiny $\pm$ 0.013}} & -0.896 {\tiny $\pm$ 0.038} & \textbf{0.897 {\tiny $\pm$ 0.002}} & 0.086 {\tiny $\pm$ 0.003} \\
EAGLE $ \lambda = 0.5 $ & -0.094 {\tiny $\pm$ 0.015} & -0.869 {\tiny $\pm$ 0.057} & 0.893 {\tiny $\pm$ 0.004} & 0.082 {\tiny $\pm$ 0.006} \\
EAGLE $ \lambda = 0.7 $ & -0.076 {\tiny $\pm$ 0.023} & -0.787 {\tiny $\pm$ 0.090} & 0.878 {\tiny $\pm$ 0.009} & 0.067 {\tiny $\pm$ 0.013} \\
EAGLE $ \lambda = 1.0 $ & -0.002 {\tiny $\pm$ 0.007} & -0.506 {\tiny $\pm$ 0.058} & 0.843 {\tiny $\pm$ 0.003} & 0.034 {\tiny $\pm$ 0.007} \\
EAGLE $ \lambda = 3.0 $ & 0.322 {\tiny $\pm$ 0.064} & 0.140 {\tiny $\pm$ 0.086} & 0.818 {\tiny $\pm$ 0.007} & 0.005 {\tiny $\pm$ 0.003} \\
EAGLE $ \lambda = 5.0 $ & 0.434 {\tiny $\pm$ 0.104} & 0.249 {\tiny $\pm$ 0.096} & 0.814 {\tiny $\pm$ 0.004} &\textbf{ 0.006 {\tiny $\pm$ 0.002}} \\
\hline
\end{tabular}
\end{table*}

\paragraph{Synthetic data.}
Our synthetic construction allows us to simulate variable levels of data separability between clients. As shown in Figure~\ref{fig:obj-gaps-synthetic}, methods that directly balance client losses assign higher weights to clients 1 and 2, steering the optimization toward their optimal models while increasing the loss for client 0. By focusing on balancing loss gaps instead of raw losses, \algoname{} avoids this bias, ensuring that clients with lower optimal training losses are not disadvantaged by the optimization. The evolution of training losses for \textit{q-FFL}, \textit{AFL}, and \textit{\algoname{}} is shown in Figure~\ref{fig:synthetic-losses} of the Appendix.

\paragraph*{EMNIST.}

Table~\ref{table:linear_alpha_01} reports detailed utility and fairness metrics for different baselines using a linear model on a highly heterogeneous data split ($\alpha = 0.1$). \algoname{} maintains accuracy comparable to other baselines. Increasing $\lambda$ gives more weight to the regularization term in the objective, producing models with smaller disparities in loss gaps but at the cost of lower utility. Beyond a certain $\lambda$, \algoname{} begins leveling down performance across all clients to further reduce the variance in loss gaps, which is expected. Results for less heterogeneous splits are provided in Section~\ref{app:more_emnist} of the Appendix.

We observe similar trends with a CNN model; detailed results are also in Section~\ref{app:more_emnist}. Figure~\ref{fig:cnn_alpha_05} illustrates the distribution of loss gaps across clients on a heterogeneous split ($\alpha = 0.5$). \algoname{} is the only algorithm that effectively reduces the variance of loss gaps, benefiting the worst-performing clients.

\paragraph{DirtyMNIST.} These experiments illustrate the scenario where clients have different optimal local losses but share the same optimal model. In this setting, the gradients from all clients point in similar directions, so methods that reweight gradients to enforce fairness behave similarly, as a model optimized for one client is also effective for the others. Detailed results are provided in Table~\ref{table:main_dirty}. Consistent with previous observations, \algoname{} improves loss-gap parity and achieves the best performance for the worst-performing client.

\section{DISCUSSION AND CONCLUSION}

In this work, we addressed the problem of client fairness in federated learning with heterogeneous data. We argued that standard loss parity can lead to severe leveling-down effects and that the notion of loss-gap parity is a more suitable alternative. To enforce this fairness criterion, we introduced \algoname{}, a regularized federated objective that explicitly controls the trade-off between utility and loss-gap parity.

Theoretically, we proved that \algoname{} converges for non-convex losses and established a connection to the standard federated learning objective through a novel heterogeneity measure. Empirically, we evaluated the approach on both synthetic and real datasets, demonstrating that it achieves the intended fairness improvements when the trade-off hyperparameter is chosen appropriately.

Some limitations of our method include the assumptions that the $\left\{L^*_k\right\}_{k=1}^K$ can be reliably estimated and that all clients remain continuously available for communication during training. These requirements limit its applicability in cross-device scenarios with constrained computation and bandwidth. Furthermore, as most fairness-promoting approaches, the method’s effectiveness depends critically on careful tuning of the hyperparameter $\lambda$, as inappropriate choices can lead to a leveling-down effect across clients.

\section*{Acknowledgments}

This work is funded by the Groupe La Poste, sponsor of the Inria Foundation, in the framework of the FedMalin Inria Challenge, and ANR-22-PESN-0014 under the France 2030 program. Michaël Perrot is supported by the French National Research Agency (ANR) through the grant ANR-23-CE23-0011-01 (Project FaCTor) and the France 2030 program with the reference ANR-23-PEIA-005 (REDEEM project).

\bibliographystyle{apalike}
\bibliography{aistats_2026}

@article{DBLP:journals/air/SalazarACA26,
  author       = {Teresa Salazar and
                  Helder Ara{\'{u}}jo and
                  Alberto Cano and
                  Pedro Henriques Abreu},
  title        = {A survey on group fairness in federated learning: challenges, taxonomy
                  of solutions and directions for future research},
  journal      = {Artif. Intell. Rev.},
  volume       = {59},
  number       = {2},
  pages        = {81},
  year         = {2026}
}

@article{benarba:hal-05093158,
  TITLE = {{Bias in Federated Learning: A Comprehensive Survey}},
  AUTHOR = {Benarba, Nawel and Bouchenak, Sara},
  JOURNAL = {{ACM Computing Surveys}},
  PUBLISHER = {{Association for Computing Machinery}},
  YEAR = {2025}
}

@inproceedings{mcmahan2017communication,
  title={Communication-efficient learning of deep networks from decentralized data},
  author={McMahan, Brendan and Moore, Eider and Ramage, Daniel and Hampson, Seth and y Arcas, Blaise Aguera},
  booktitle={Artificial intelligence and statistics},
  pages={1273--1282},
  year={2017},
  organization={PMLR}
}

@inproceedings{lifair,
  title={Fair Resource Allocation in Federated Learning},
  author={Li, Tian and Sanjabi, Maziar and Beirami, Ahmad and Smith, Virginia},
  booktitle={International Conference on Learning Representations},
  year={2020}
}

@inproceedings{li2021dittofairrobustfederated,
      title={Ditto: Fair and Robust Federated Learning Through Personalization}, 
      author={Tian Li and Shengyuan Hu and Ahmad Beirami and Virginia Smith},
      year={2021},
      booktitle={International Conference on Machine Learning},
      url={http://proceedings.mlr.press/v139/li21h/li21h-supp.pdf}, 
}

@article{wang2021field,title	= {A Field Guide to Federated Optimization},author	= {Jianyu Wang and Zachary Burr Charles and Zheng Xu and Gauri Joshi and Brendan McMahan and Blaise Hilary Aguera-Arcas and Maruan Al-Shedivat and Galen Andrew and A. Salman Avestimehr and Katharine Daly and Deepesh Data and Suhas Diggavi and Hubert Eichner and Advait Gadhikar and Zachary Garrett and Antonious M. Girgis and Filip Hanzely and Andrew Hard and Chaoyang He and Samuel Horvath and Zhouyuan Huo and Alex Ingerman and Martin Jaggi and Tara Javidi and Peter Kairouz and Satyen Chandrakant Kale and Sai Praneeth Karimireddy and Jakub Konečný and Sanmi Koyejo and Tian Li and Luyang Liu and Mehryar Mohri and Hang Qi and Sashank Reddi and Peter Richtarik and Karan Singhal and Virginia Smith and Mahdi Soltanolkotabi and Weikang Song and Ananda Theertha Suresh and Sebastian Stich and Ameet Talwalkar and Hongyi Wang and Blake Woodworth and Shanshan Wu and Felix Yu and Honglin Yuan and Manzil Zaheer and Mi Zhang and Tong Zhang and Chunxiang (Jake) Zheng and Chen Zhu and Wennan Zhu},year	= {2021},journal={arXiv preprint arXiv:2107.06917}}

@inproceedings{chu2023focusfairnessagentawarenessfederated,
  title={FOCUS: Fairness via Agent-Awareness for Federated Learning on Heterogeneous Data},
  author={Chu, Wenda and Xie, Chulin and Wang, Boxin and Li, Linyi and Yin, Lang and Nourian, Arash and Zhao, Han and Li, Bo},
  booktitle={International Workshop on Federated Learning in the Age of Foundation Models in Conjunction with NeurIPS 2023},
  year = {2023}
}

@inproceedings{mohri2019agnosticfederatedlearning,
  title={Agnostic federated learning},
  author={Mohri, Mehryar and Sivek, Gary and Suresh, Ananda Theertha},
  booktitle={International conference on machine learning},
  pages={4615--4625},
  year={2019},
  organization={PMLR}
}

@inproceedings{mohri2019agnostic,
  title={Agnostic federated learning},
  author={Mohri, Mehryar and Sivek, Gary and Suresh, Ananda Theertha},
  booktitle={International conference on machine learning},
  pages={4615--4625},
  year={2019},
  organization={PMLR}
}

@inproceedings{cohen2017emnist,
  title={EMNIST: Extending MNIST to handwritten letters},
  author={Cohen, Gregory and Afshar, Saeed and Tapson, Jonathan and Van Schaik, Andre},
  booktitle={2017 international joint conference on neural networks (IJCNN)},
  pages={2921--2926},
  year={2017},
  organization={IEEE}
}

@article{Yue_2023,
   title={GIFAIR-FL: A Framework for Group and Individual Fairness in Federated Learning},
   volume={2},
   ISSN={2694-4030},
   url={http://dx.doi.org/10.1287/ijds.2022.0022},
   DOI={10.1287/ijds.2022.0022},
   number={1},
   journal={INFORMS Journal on Data Science},
   publisher={Institute for Operations Research and the Management Sciences (INFORMS)},
   author={Yue, Xubo and Nouiehed, Maher and Al Kontar, Raed},
   year={2023},
   month=apr, pages={10–23} }

@article{rieke2020future,
  title={The future of digital health with federated learning},
  author={Rieke, Nicola and Hancox, Jonny and Li, Wenqi and Milletari, Fausto and Roth, Holger R and Albarqouni, Shadi and Bakas, Spyridon and Galtier, Mathieu N and Landman, Bennett A and Maier-Hein, Klaus and others},
  journal={NPJ digital medicine},
  volume={3},
  number={1},
  pages={119},
  year={2020},
  publisher={Nature Publishing Group UK London}
}

@incollection{long2020federated,
  title={Federated learning for open banking},
  author={Long, Guodong and Tan, Yue and Jiang, Jing and Zhang, Chengqi},
  booktitle={Federated learning: privacy and incentive},
  pages={240--254},
  year={2020},
  publisher={Springer}
}

@article{li2020review,
  title={A review of applications in federated learning},
  author={Li, Li and Fan, Yuxi and Tse, Mike and Lin, Kuo-Yi},
  journal={Computers \& Industrial Engineering},
  volume={149},
  pages={106854},
  year={2020},
  publisher={Elsevier}
}

@article{hsu2019measuring,
  title={Measuring the effects of non-identical data distribution for federated visual classification},
  author={Hsu, Tzu-Ming Harry and Qi, Hang and Brown, Matthew},
  journal={arXiv preprint arXiv},
  eprint={1909.06335},
  url = {https://arxiv.org/abs/1909.06335},
  year={2019}
}

@article{kairouz2021advances,
  title={Advances and open problems in federated learning},
  author={Kairouz, Peter and McMahan, H Brendan and Avent, Brendan and Bellet, Aur{\'e}lien and Bennis, Mehdi and Bhagoji, Arjun Nitin and Bonawitz, Kallista and Charles, Zachary and Cormode, Graham and Cummings, Rachel and others},
  journal={Foundations and trends{\textregistered} in machine learning},
  volume={14},
  number={1--2},
  pages={1--210},
  year={2021},
  publisher={Now Publishers, Inc.}
}

@inproceedings{
cho2022to,
title={To Federate or Not To Federate: Incentivizing Client Participation in Federated Learning},
author={Yae Jee Cho and Divyansh Jhunjhunwala and Tian Li and Virginia Smith and Gauri Joshi},
booktitle={Workshop on Federated Learning: Recent Advances and New Challenges (in Conjunction with NeurIPS 2022)},
year={2022},
url={https://openreview.net/forum?id=pG08eM0CQba}
}

@inproceedings{zhang2012some,
  title={Some new deformation formulas about variance and covariance},
  author={Zhang, Yuli and Wu, Huaiyu and Cheng, Lei},
  booktitle={2012 proceedings of international conference on modelling, identification and control},
  pages={987--992},
  year={2012},
  organization={IEEE}
}

@article{meng2024gradient,
  title={Gradient descent on logistic regression with non-separable data and large step sizes},
  author={Meng, Si Yi and Orvieto, Antonio and Cao, Daniel Yiming and De Sa, Christopher},
  journal={arXiv preprint arXiv:2406.05033},
  year={2024}
}

@inproceedings{wang2020tackling,
  title={Tackling the objective inconsistency problem in heterogeneous federated optimization},
author={Wang, Jianyu and Liu, Qinghua and Liang, Hao and Joshi, Gauri and Poor, H Vincent},
  booktitle={Advances in neural information processing systems},
  volume={33},
  pages={7611--7623},
  year={2020}
}

@article{mittelstadt2023unfairness,
  title={The unfairness of fair machine learning: Leveling down and strict egalitarianism by default},
  author={Mittelstadt, Brent and Wachter, Sandra and Russell, Chris},
  journal={Mich. Tech. L. Rev.},
  volume={30},
  pages={1},
  year={2023},
  publisher={HeinOnline}
}

@inproceedings{zietlow2022leveling,
  title={Leveling down in computer vision: Pareto inefficiencies in fair deep classifiers},
  author={Zietlow, Dominik and Lohaus, Michael and Balakrishnan, Guha and Kleindessner, Matth{\"a}us and Locatello, Francesco and Sch{\"o}lkopf, Bernhard and Russell, Chris},
  booktitle={Proceedings of the IEEE/CVF Conference on Computer Vision and Pattern Recognition},
  pages={10410--10421},
  year={2022}
}

@inproceedings{maheshwari2023fair,
  title={Fair Without Leveling Down: A New Intersectional Fairness Definition},
  author={Maheshwari, Gaurav and Bellet, Aur{\'e}lien and Denis, Pascal and Keller, Mikaela},
  year = {2023},
  booktitle={The 2023 Conference on Empirical Methods in Natural Language Processing}
}

@inproceedings{yu2019parallel,
  title={Parallel restarted SGD with faster convergence and less communication: Demystifying why model averaging works for deep learning},
  author={Yu, Hao and Yang, Sen and Zhu, Shenghuo},
  booktitle={Proceedings of the AAAI conference on artificial intelligence},
  pages={5693--5700},
  year={2019}
}

@article{carey2025achieving,
  title={Achieving Distributive Justice in Federated Learning via Uncertainty Quantification},
  author={Carey, Alycia and Wu, Xintao},
  journal={arXiv preprint arXiv:2504.15924},
  eprint={2504.15924},
  year={2025},
  url={https://arxiv.org/abs/2504.15924}
}

@article{shi2023towards,
  title={Towards fairness-aware federated learning},
  author={Shi, Yuxin and Yu, Han and Leung, Cyril},
  journal={IEEE Transactions on Neural Networks and Learning Systems},
  volume={35},
  number={9},
  pages={11922--11938},
  year={2023},
  publisher={IEEE}
}

@inproceedings{mukhoti2023deep,
  title={Deep deterministic uncertainty: A new simple baseline},
  author={Mukhoti, Jishnu and Kirsch, Andreas and Van Amersfoort, Joost and Torr, Philip HS and Gal, Yarin},
  booktitle={Proceedings of the IEEE/CVF Conference on Computer Vision and Pattern Recognition},
  pages={24384--24394},
  year={2023}
}

\section*{Checklist}

\begin{enumerate}

  \item For all models and algorithms presented, check if you include:
  \begin{enumerate}
    \item A clear description of the mathematical setting, assumptions, algorithm, and/or model. [Yes]
    \item An analysis of the properties and complexity (time, space, sample size) of any algorithm. [Yes]
    \item (Optional) Anonymized source code, with specification of all dependencies, including external libraries. [Yes] (Submitted after acceptance)
  \end{enumerate}

  \item For any theoretical claim, check if you include:
  \begin{enumerate}
    \item Statements of the full set of assumptions of all theoretical results. [Yes]
    \item Complete proofs of all theoretical results. [Yes]
    \item Clear explanations of any assumptions. [Yes]     
  \end{enumerate}

  \item For all figures and tables that present empirical results, check if you include:
  \begin{enumerate}
    \item The code, data, and instructions needed to reproduce the main experimental results (either in the supplemental material or as a URL). [Yes]
    \item All the training details (e.g., data splits, hyperparameters, how they were chosen). [Yes]
    \item A clear definition of the specific measure or statistics and error bars (e.g., with respect to the random seed after running experiments multiple times). [Yes]
    \item A description of the computing infrastructure used. (e.g., type of GPUs, internal cluster, or cloud provider). [Yes]
  \end{enumerate}

  \item If you are using existing assets (e.g., code, data, models) or curating/releasing new assets, check if you include:
  \begin{enumerate}
    \item Citations of the creator If your work uses existing assets. [Yes]
    \item The license information of the assets, if applicable. [Not Applicable]
    \item New assets either in the supplemental material or as a URL, if applicable. [Not Applicable]
    \item Information about consent from data providers/curators. [Not Applicable]
    \item Discussion of sensible content if applicable, e.g., personally identifiable information or offensive content. [Not Applicable]
  \end{enumerate}

  \item If you used crowdsourcing or conducted research with human subjects, check if you include:
  \begin{enumerate}
    \item The full text of instructions given to participants and screenshots. [Not Applicable]
    \item Descriptions of potential participant risks, with links to Institutional Review Board (IRB) approvals if applicable. [Not Applicable]
    \item The estimated hourly wage paid to participants and the total amount spent on participant compensation. [Not Applicable]
  \end{enumerate}

\end{enumerate}

\clearpage
\appendix

\appendix
\thispagestyle{empty}

\onecolumn
\aistatstitle{Supplementary Material}
\section{DETAILED THEORETICAL RESULTS AND PROOFS}\label{appendix:proofs}

In this appendix, we provide the convergence proof for \algoname{}. To facilitate understanding of the proof, all notations are defined in Table~\ref{tab:notations}.

Recall that \algoname{} optimizes the following objective:
\begin{align}\label{eq:regularized-obj-appendix}
 &\arg\min\limits_{\theta \in \mathcal{H}} F(\theta)  := \frac{1}{K}\sum_{k=1}^K  \bigg[L_{k}(\theta) +  \frac{\lambda}{K-1} \sumkkp  (r_{k}(\theta) - r_{k'}(\theta))^2\bigg],
\end{align}
 where $r_{k}(\theta) := L_{k}(\theta) - L_{k}^* $, and $\lambda > 0$ is a regularization parameter to be tuned in practice.

Our proof adapts the analysis of FedAvg in \citet{yu2019parallel}, where we establish an upper bound on the smoothness parameter and a bound on the gradient norm for our new regularized objective, following a similar proof strategy.

 We first observe that the gradient of $F$ is equal to :
\begin{align*}
    \nabla F(\theta) := \frac{1}{K}\sum_{k=1}^K \bigg( \underbrace{1+ \frac{4\lambda}{K-1}  \sumkkp (r_{k}(\theta) - r_{k'}(\theta))}_{w_{k}(\theta)} \bigg) \nabla L_{k}(\theta)
\end{align*}

We introduce a new ``proxy'' gradient 
\begin{align*}
\nabla \tilde{F}_{k}(\theta, \thetapr) := \bigg( \underbrace{1+ \frac{4\lambda}{K-1}  \sumkkp (r_{k}(\theta') - r_{k'}(\theta'))}_{w_k(\theta')} \bigg) \nabla L_{k}(\theta)
\end{align*} 
Here, the functions $r_k$ and $r_{k'}$ are evaluated on a different model than the stochastic gradient $\nabla L_k$, this will remove the need for communication between synchronization rounds.

We will use two time counters, $t$ to denote the the number of synchronizations and $\tau$ to index the local step, and we will use $k$ in subscript to index the client. For example, $\theta^{(t, \tau)}_k$ denotes the model of client $k$ after $t$ synchronizations and $\tau$ local steps. 
 Additionally, we denote a virtual average model  $\bar{\theta}^{(t, \tau)} := \frac{1}{K}\sum_{k=1}^K \theta^{(t, \tau)}_k$ to represent the average of local models within clients between two rounds communication. 
 For all the results mentioned in this section, the expectation $\mathbb{E}$ with respect to some function of the model $\btheta^{(t, \tau)}$ is taken with respect to the minibatching noises $\zeta_1$, $\zeta_2$, $\dots$, $\zeta_k$ at all times up to $(t, \tau)$.
 
 With this notation, we have: $\theta^{(t+1, 0)}_k := \frac{1}{K}\sum_{k=1}^K \theta^{(t, I)}_k = \bar{\theta}^{(t, I)} $. In words, at the start of each local training phase, the local models are initialized with the average of the models of all clients.
 
 We can derive a recursive update on the virtual model: 
 \begin{align}
 \bar{\theta}^{(t, \tau + 1)} &= \frac{1}{K}\sum_{k = 1}^K \theta_k^{(t, \tau + 1)} \\
 &= \frac{1}{K}\sum_{k=1}^K \theta_k^{(t, \tau)} - \eta \nabla \tilde{F}_k(\theta_k^{(t, \tau)}, \bar{\theta}^{(t,0)}; \zeta_k) \\&= \bar{\theta}^{(t, \tau)} - \eta \frac{1}{K}\sum_{k=1}^K \nabla \tilde{F}_k(\theta_k^{(t, \tau)}, \bar{\theta}^{(t,0)}; \zeta_k).
 \end{align}

\begin{table}[t]
\centering
\begin{tabular}{|c|l|}
\hline
\textbf{Notation} & \textbf{Description} \\
\hline
$\lambda$ & Regularization parameter to be controlled by the user \\
$\theta$ & An arbitrary model in the space \\
$t$ & counter of synchronization rounds\\
$\tau$ & counter of local steps\\
$\theta^{(t, \tau)}_k$ & Local model inside each client $k$, after $t$ syncs and $\tau$ local steps. \\
$\btheta^{(t, \tau)}$ &  $\frac{1}{K} \sumk \theta^{(t, \tau)}_k$\\
$K$ & Number of clients \\
$L_k(\theta)$ & Loss of client $k$ \\

$F_k(\theta)$ & $L_k(\theta) + \frac{\lambda}{K-1} \sumkkp (r_k(\theta) - r_{k'}(\theta))^2$\\
$\nabla L_k(\theta)$ & Gradient of $L_k$ \\
$\nabla L_k(\theta,\zeta_k)$ & (Minibatch) Stochastic Gradient of $L_k$ \\
$\sigma$ & Bound on the variance of stochastic gradients of $L_k$ \\
$\beta$ & Smoothness parameter of $L_k$ \\
$\Gamma$ & Heterogeneity measure, $\Gamma := \sup\limits_{\theta \in \mathcal{H}}\max\limits_{k , k'} |r_k(\theta) - r_{k'}(\theta)|$ \\
$B$ & Bound on $\| \nabla L_k(\theta; \zeta_k) \|^2$\\
$w_k(\theta)$ & $(1 + 4\frac{\lambda}{K-1} \sumkkp (r_k(\theta) - r_{k'}(\theta) )) $ \\
$\nabla \tilde{F}_k(\theta)$ & $w_k(\theta) \nabla L_k(\theta)$ \\
$\nabla \tilde{F}_k(\theta, \thetapr)$ & $w_k(\thetapr) \nabla L_k(\theta)$ \\
\hline
\end{tabular}
\caption{Table of notations and their descriptions}
\label{tab:notations}
\end{table}

We will first introduce lemmas that are useful in our proof of our main results. Then, we will restate the results of the main paper and provide their proofs. 

As a preliminary step, we establish bounds on the weights $w_k(\theta), w_k(\thetapr)$, which will be useful in the subsequent analysis.
\begin{lemma}[Properties of the weights of clients $w_k(\theta)$]\label{lemma:w_k}
For any two models $\theta, \thetapr \in \mathcal{H}$, the weights $w_k(\theta), w_k(\thetapr)$ satisfy the following:
\begin{align}
|w_k(\theta)| &\leq 1 + 4 \lambda \Gamma,\\
|w_k(\theta) - w_k(\thetapr )| &\leq 8 \lambda B \| \theta - \thetapr \|.
\end{align}

\end{lemma}

\begin{proof}
Indeed, for a $\theta \in \mathcal{H}$ we have for the first property:
\begin{align}
|w_k(\theta)| &= \bigg|1 + 4\frac{\lambda}{K-1} \sumkkp (r_k(\theta) - r_{k'}(\theta) )\bigg|,\\
&\stackrel{\text{Triangle inequality}}{\leq} 1 + 4 \frac{\lambda}{K - 1} \sumkkp |r_k(\theta) - r_{k'}(\theta)|, \\
&\stackrel{|r_k(\theta) - r_{k'}(\theta)| \leq \Gamma}{\leq} 1 + 4 \lambda \Gamma.
\end{align}
And for a $\theta, \theta' \in \mathcal{H}$ we have for the second property:
\begin{align}
    |w_k(\theta) - w_k(\theta')| &\leq \bigg|4\frac{\lambda}{K-1} \sumkkp (r_k(\theta) - r_{k'}(\theta) ) -  4\frac{\lambda}{K-1} \sumkkp (r_k(\thetapr) - r_{k'}(\thetapr) )\bigg|,\\
    &\stackrel{\text{Triangular inequality}}{\leq} 4 \frac{\lambda}{K - 1} \sumkkp |r_k(\theta) - r_{k}(\thetapr) + r_{k'}(\thetapr) - r_{k'}(\theta) |,\\
    &\stackrel{\text{$L_k$ is $B$-Lipschitz}}{\leq} 8\lambda B \| \theta - \thetapr \|.\qedhere
\end{align}
\end{proof}

The following lemma derives an upper on bound on the norm of $\| \nabla \tilde{F}_k(\theta, \thetapr; \zeta_k) \|$:

\begin{lemma}[Bounded norm of the gradient of $\tilde{F}_k$] \label{eq:bound_on_F}
Under the assumption that the norm of the stochastic gradient of $L_k$ is bounded by a constant $B$, for any two models $\theta, \thetapr \in \mathcal{H}$, we have:
\begin{align}
\| \nabla \tilde{F}_k(\theta, \thetapr; \zeta_k) \| \leq (1 + 4 \lambda \Gamma) B.
\end{align}
\end{lemma}
\begin{proof}
We have $\forall \theta, \thetapr \in \mathcal{H}$:
\begin{align}
	\| \nabla F_k(\theta, \thetapr; \zeta_k) \| &= \bigg\| w_k(\thetapr) \nabla L_k(\theta; \zeta_k) \bigg\|
\\
&= |w_k(\thetapr)| \cdot \|\nabla L_k(\theta; \zeta_k)\|\\
&\stackrel{\mbox{Lemma }~\ref{lemma:w_k}, \|\nabla L_k(\theta; \zeta_k)\| \leq B}{\leq} (1 + 4 \lambda \Gamma) B.\qedhere
\end{align}

\end{proof}

The second lemma derives an upper bound on the smoothness parameter of the regularized local objective $F_k(\theta) := L_k(\theta) + \frac{\lambda}{K-1} \sumkkp (r_k(\theta) - r_{k'}(\theta))^2$:
\begin{lemma}\label{lemma:smoothness_F}
Under the assumptions~\ref{assumptions} $, \forall k \in [K]: F_k$ is smooth with a smoothness parameter equal to: $(1 + 4\lambda \Gamma  )\beta + 8\lambda B^2$.
\begin{align}
\forall \theta \in \mathcal{H}: \| \nabla^2 F_k(\theta) \| \leq (1 + 4\lambda \Gamma  )\beta + 8\lambda B^2.
\end{align}
\end{lemma}
\begin{proof}

We know that if $F_k$ is twice differentiable, we have: 
\begin{align}
F_k \mbox{ is } \beta' \mbox{-smooth} &\iff \forall \theta \in \mathcal{H}:  \| \nabla^2 F_k(\theta) \| \leq \beta',
\end{align}
Let's derive an upper bound on $\beta'$.
We have: 
\begin{align*}
&\nabla^2 F_k(\theta) = \frac{\partial }{\partial \theta} \bigg[w_k(\theta) \nabla L_k(\theta)\bigg]\\
&= w_k(\theta) \nabla^2 L_k(\theta) + 4 \frac{\lambda}{K-1} \sumkkp (\nabla L_{k}(\theta) - \nabla L_{k'}(\theta))\nabla L_k(\theta)^T 
\end{align*}
Thus: 
\begin{align}
\|  \nabla^2 F_k(\theta) \| &= \bigg\|w_k(\theta)  \nabla^2 L_k(\theta) + 4 \frac{\lambda}{K-1} \sumkkp (\nabla L_k(\theta) - \nabla L_{k'}(\theta))\nabla L_k(\theta)^T  \bigg\|,\\
\|  \nabla^2 F_k(\theta) \| &\leq | w_k(\theta) |  \underbrace{\|  \nabla^2 L_k(\theta) \|}_{\leq \beta, \mbox{since $L_k$ is $\beta$-smooth}} \nonumber\\&+ \bigg\| 4 \frac{\lambda}{K-1} \sumkkp (\nabla L_k(\theta) - \nabla L_{k'}(\theta))\nabla L_k(\theta)^T  \bigg\|,\\
&\stackrel{\mbox{Lemma }~\ref{lemma:w_k}}{\leq}  (1 + 4 \lambda \Gamma)  \beta + \bigg\| 4 \frac{\lambda}{K-1} \sumkkp (\nabla L_k(\theta) - \nabla L_{k'}(\theta))\nabla L_k(\theta)^T  \bigg\|,\\
&\leq (1 + 4\lambda \Gamma) \beta  \nonumber \\&+ 4 \frac{\lambda}{K - 1} \sumkkp (\|(\nabla L_k(\theta)\nabla L_{k'}(\theta)^T \| + \|\nabla L_{k}(\theta)\nabla L_{k}(\theta)^T \|),\\
&\stackrel{\|A B\| \leq \|A\| \|B\| }\leq (1 + 4\lambda \Gamma) \beta \nonumber \\&+ 4 \frac{\lambda}{K-1} \sumkkp (\|(\nabla L_k(\theta)\|\|\nabla L_k(\theta)^T \| + \|\nabla L_{k'}(\theta)\|\|\nabla L_k(\theta)^T \|),\\
&\stackrel{\| \nabla L_k(\theta; \zeta_k) \| \leq B}{\leq} \underbrace{\bigg((1 + 4\lambda \Gamma  )\beta + 8\lambda B^2\bigg) }_{\beta'}. \label{smoothness-F}
\end{align}
\end{proof}

The next lemma bounds the difference between the local models and the averaged model.

\begin{lemma}[Bounded distance of local and average models] \label{eq:diff_local_average}
After $t$ synchronization steps, and $\tau$ local steps, under assumptions~\ref{assumptions}, the algorithm~\ref{alg:algoname} satisfies:
\begin{align} 
\E[\| \bar{\theta}^{(t, \tau)} - \theta_k^{t, \tau} \|^2] \leq 4 \eta^2 I^2 (1 + 4\lambda \Gamma)^2 B^2.
\end{align}
\end{lemma}
\begin{proof}

We have: 
\begin{align}
    \bar{\theta}^{(t, \tau)} &= \frac{1}{K}\sum_{k = 1}^K \theta_k^{(t, \tau)}\\
    &= \frac{1}{K}\sum_{k=1}^K \theta_k^{(t, \tau - 1)} - \eta \nabla \tilde{F}_k(\theta_k^{(t, \tau - 1)}, \bar{\theta}^{(t,0)}; \zeta_k)\\
    &= \bar{\theta}^{(t, \tau - 1)} - \eta \frac{1}{K}\sum_{k=1}^K \nabla \tilde{F}_k(\theta_k^{(t, \tau - 1)}, \bar{\theta}^{(t,0)}; \zeta_k)\\
    &= \bar{\theta}^{(t, 0)} - \eta \frac{1}{K}\sum_{k=1}^K\sum_{\tau' = 0}^{\tau - 1} \nabla \tilde{F}_k(\theta_k^{(t, \tau')}, \bar{\theta}^{(t,0)}; \zeta_k)
\end{align}

And similarly: 
\begin{align}
    \theta_k^{(t, \tau)} &= \theta_k^{(t, \tau - 1)}  - \eta \nabla \tf_k(\theta_k^{(t, \tau - 1)}, \bar{\theta}_k^{(t, 0)}; \zeta_k)\\
    &= \bar{\theta}^{(t, 0)} - \eta \sum_{\tau' = 0}^{\tau - 1} \nabla \tf_k(\theta_k^{(t,\tau')}, \bar{\theta}^{(t,0)}; \zeta_k) 
\end{align}

Therefore:
\begin{align}
&\E[\| \bar{\theta}^{(t, \tau)} - \theta_k^{(t, \tau)} \|^2] \nonumber\\&= \eta^2 \E \bigg[ \bigg\| \sum_{\tau' = 0}^{\tau - 1} \frac{1}{K}\bigg(\sum_{k'=1}^K \nabla \tf_{k'}(\theta_{k'}^{(t, \tau')} , \bar{\theta}^{(t, 0)}; \zeta_{k'})\bigg) - \sum_{\tau' = 0}^{\tau - 1}  \nabla \tf_{k}(\theta_{k}^{(t, \tau')} , \bar{\theta}^{(t, 0)}; \zeta_{k}) \bigg\|^2\bigg]\\
&\leq 2 \eta^2 \tau \bigg(\sum_{\tau' = 0}^{\tau - 1} \E\bigg[\bigg\| \frac{1}{K}(\sum_{k'=1}^K \nabla \tf_{k'}(\theta_{k'}^{(t, \tau')} , \bar{\theta}^{(t, 0)}; \zeta_{k'}))\bigg\|^2\bigg] + \E\bigg[\bigg\| \nabla \tf_{k}(\theta_{k}^{(t, \tau')} , \bar{\theta}^{(t, 0)}; \zeta_{k}) \bigg\|^2 \bigg] \bigg)\\
&\stackrel{\mbox{Lemma }~\ref{eq:bound_on_F}}{\leq} 4 \eta^2 \tau^2 (1 + 4\lambda \Gamma)^2 B^2\\
&\stackrel{\tau < I}{\leq} 4 \eta^2 I^2 (1 + 4\lambda \Gamma)^2 B^2.\qedhere
\end{align}
\end{proof}
\setcounter{theorem}{\value{savedtheoremone}}
\addtocounter{theorem}{-1}

\begin{theorem}[Convergence to a solution of \eqref{eq:regularized-obj}]
\label{theorem:conv_F_appendix}
Let $\theta_k^{(t, \tau)}$ refer to the model of client $k$ after $t$ communication rounds and $\tau$ local steps and let $\btheta^{(t, \tau)} := \frac{1}{K} \sumk \theta_k^{(t, \tau)}$. Let $T$ be the total number of communication rounds and $I$ be the number of local steps between each communication round. Under Assumption~\ref{assumptions}, for $\eta \leq \frac{1}{(1+4\lambda \Gamma)\beta + 8 \lambda B^2}$, the sequence of models $\{\btheta^{(t, \tau)}\}_{t \geq 0, \tau \geq 0}$ generated by Algorithm~\ref{alg:algoname} satisfies:
\begin{align}
&\frac{1}{T I} \sum_{t = 1}^T \sum_{\tau = 1}^I \E[\| \nabla F(\btheta^{(t,\tau)}) \|^2] \leq 2\frac{1}{\eta T I} (F(\btheta^{(1,1)}) -  F^*) +  \eta^2 I^2 \xi_1 +\eta  \frac{\sigma^2}{K}  \xi_2 , \nonumber
\end{align} 
with $F^{*} := \arg\min\limits_{\theta \in \mathcal{H}} F(\theta),  \xi_1 :=  2 B^2 (1 + 4\lambda \Gamma)^2 (\beta^2 + 32 \lambda^2  (   \Gamma^2 \beta^2   +   B^4   ))$ and $\xi_2 := ((1 + 4\lambda \Gamma  )\beta + 8\lambda B^2)(1 + 4\lambda \Gamma)^2$.

If we further choose  $\eta =  \frac{1}{\sqrt{TI}}$ and $1 \leq I \leq \sqrt{T}$, the rate can be shown to be sublinear in $T$, that is:
\begin{align}
&\frac{1}{T I} \sum_{t = 1}^T \sum_{\tau = 1}^I \E[\| \nabla F(\btheta^{(t,\tau)}) \|^2] \leq \underbrace{2\frac{1}{\sqrt{IT}} (F(\btheta^1) -  F^*) +   \frac{I}{T} \xi_1 +\frac{\xi_2}{\sqrt{IT}} \frac{\sigma^2}{K}.}_{\mathcal{O}(\frac{1}{\sqrt{T}})} \nonumber
\end{align} 
\end{theorem}
\begin{proof}

First, let us start by upper bounding the following quantity (the difference of the applied gradient to the actual gradient of $\nabla F_k$):
\begin{align}
    &\|\nabla \tf_k(\theta_k^{(t, \tau)}, \btheta^{(t, 0)}) - \nabla F_k(\btheta^{(t, \tau)}) \|^2  = \bigg\| w_k(\btheta^{(t, 0)}) \nabla L_k(\theta_k^{(t, \tau)}) - w_k(\btheta^{(t, \tau)}) \nabla L_k(\btheta^{(t, \tau)})  \bigg\|^2\\
    &= \| w_k(\btheta^{(t, 0)}) (\nabla L_k(\theta_k^{(t, \tau)}) - \nabla L_k(\btheta^{(t, \tau)})) + (w_k(\btheta^{(t, 0)}) - w_k(\btheta^{(t, \tau)})) \nabla L_k(\btheta^{(t, \tau)})  \|^2\\
    &\leq \bigg[(1 + 4 \lambda \Gamma) \| \nabla L_k(\theta_k^{(t, \tau)}) - \nabla L_k(\btheta^{(t, \tau)}) \| + |w_k(\btheta^{(t, 0)}) - w_k(\btheta^{(t, \tau)})| B \bigg]^2\\
    &\stackrel{\text{$L_k$ is $B$-smooth, Lemma~\ref{lemma:w_k}}}{\leq} \bigg[(1 + 4 \lambda \Gamma) \beta \| \theta_k^{(t, \tau)} - \btheta^{(t, \tau)} \| + 8 \lambda B^2 \| \btheta^{(t, 0)} - \btheta^{(t, \tau)} \|  \bigg]^2\\
    &\stackrel{(a+b)^2 \leq 2 a^2 + 2 b^2}{\leq} 2 \bigg[(1 + 4 \lambda \Gamma)^2 \beta^2 \| \theta_k^{(t, \tau)} - \btheta^{(t, \tau)} \|^2 + 64 \lambda^2 B^4 \| \btheta^{(t, 0)} - \btheta^{(t, \tau)} \|^2  \bigg]\\
    &\stackrel{\text{Lemma~\ref{eq:diff_local_average} and~\ref{eq:bound_on_F}}}{\leq} 2 \bigg[(1 + 4 \lambda \Gamma)^2 \beta^2 4 \eta^2 I^2 (1 + 4 \lambda \Gamma)^2 B^2+ 64 \lambda^2 B^4 2 \eta^2 I^2 (1 + 4 \lambda)^2 B^2  \bigg]\\
    &= 8 \eta^2 I^2 (1 + 4 \lambda)^2 B^2 \bigg[(1 + 4 \lambda \Gamma)^2 \beta^2  + 32 \lambda^2 B^4  \bigg].
\end{align}
Therefore we have:
\begin{align}
&\bigg\| \frac{1}{K} \sumk \bigg[\nabla \tf_k(\btheta^{(t, \tau)}, \btheta^{(t, \tau)}) - \nabla \tf_k(\theta_k^{(t, \tau)}, \btheta^{(t, 0)}) \bigg] \bigg\|^2 \nonumber \\&\leq \frac{1}{K}  \sumk \|  \nabla \tf_k(\btheta^{(t, \tau)}, \btheta^{(t, \tau)}) - \nabla \tf_k(\theta_k^{(t, \tau)}, \btheta^{(t, 0)}) \|^2.\\
&\leq 8 \eta^2 I^2 (1 + 4 \lambda)^2 B^2 \bigg[(1 + 4 \lambda \Gamma)^2 \beta^2  + 32 \lambda^2 B^4  \bigg]\label{line:bound-gradient-correction}.
\end{align}

Since each $F_k$ is smooth with smoothness equal to $(1 + 4\lambda \Gamma  )\beta + 8\lambda B^2$, 
$F(\theta) := \frac{1}{K}\sumk F_k(\theta)$ is also $(1 + 4\lambda \Gamma )\beta + 8\lambda B^2$-smooth. 
Therefore:
\begin{align}
	&\E[F(\btheta^{(t, \tau + 1)})] \nonumber \\
    &\leq \E[F(\btheta^{(t, \tau)})] + \E[\langle\nabla F(\btheta^{(t, \tau)}), \btheta^{(t , \tau + 1 )} - \btheta^{(t, \tau)} \rangle] + \frac{(1 + 4\lambda \Gamma  )\beta + 8\lambda B^2}{2} \E[\|\btheta^{(t, \tau + 1)} - \btheta^{(t, \tau)} \|^2]\\
	&= \E[F(\btheta^{(t, \tau)})] - \eta \E\bigg[\bigg\langle\nabla F(\btheta^{(t, \tau)}), \frac{1}{K}\sumk \nabla \tf_k(\theta_k^{(t, \tau)}, \btheta^{(t, 0)} ; \zeta_k) \bigg\rangle\bigg] \nonumber\\& \qquad+ \frac{(1 + 4\lambda \Gamma )\beta + 8\lambda B^2}{2} \eta^2 \E\bigg[\bigg\| \frac{1}{K}\sumk \nabla \tf_k(\theta_k^{(t, \tau)}, \btheta^{(t, 0)} ; \zeta_k) \bigg\|^2\bigg]\\
	&\stackrel{\E_{\zeta_k}[\nabla \tilde{F}_k(\theta, \theta'; \zeta_k)] = \nabla \tilde{F}_k(\theta, \theta')}{=} \E[F(\btheta^{(t, \tau)})] - \eta \E \bigg\langle\nabla F(\btheta^{(t, \tau)}), \frac{1}{K}\sumk \nabla \tf_k(\theta_k^{(t, \tau)}, \btheta^{(t, 0)}) \bigg\rangle \nonumber\\& \qquad+ \frac{(1 + 4\lambda \Gamma  )\beta + 8\lambda B^2}{2} \eta^2 \E\bigg[\bigg\| \frac{1}{K}\sumk \nabla \tf_k(\theta_k^{(t, \tau)}, \btheta^{(t, 0)} ; \zeta_k) \bigg\|^2\bigg]
\end{align}

We have:
\begin{align}
&\eta^2 \E\bigg[\bigg\| \frac{1}{K}\sumk \nabla \tf_k(\theta_k^{(t, \tau)}, \btheta^{(t, 0)} ; \zeta_k) \bigg\|^2\bigg]\\ 
&\stackrel{(a)}{=} \eta^2 \E[\|  \frac{1}{K}\sumk \nabla \tf_k(\theta_k^{(t, \tau)}, \btheta^{(t, 0)} ; \zeta_k) - \nabla \tf_k(\theta_k^{(t, \tau)}, \btheta^{(t, 0)})  \|^2] + \eta^2 \|  \frac{1}{K}\sumk \nabla \tf_k(\theta_k^{(t, \tau)}, \btheta^{(t, 0)}) \|^2 \\
&= \eta^2 \| \frac{1}{K} \sumk \nabla \tf_k(\theta_k^{(t, \tau)}, \btheta^{(t, 0)}) \|^2  + \nonumber\\&  \eta^2 \E \bigg[\bigg\|\frac{1}{K} \sumk \bigg(1 + \frac{4 \lambda}{K-1} \sumkkp (r_k(\btheta^{(t, 0)}) - r_{k'}(\btheta^{(t, 0)}))\bigg) (\nabla L_k(\theta_k^{(t, \tau)}; \zeta_k) - \nabla L_k(\theta_k^{(t, \tau)}))\bigg\|^2\bigg] \\
&\stackrel{(b)}{=} \eta^2   \frac{1}{K^2}\sumk \bigg|1 + \frac{4 \lambda}{K-1} \sumkkp (r_k(\btheta_k^{(t, 0)}) - r_{k'}(\btheta^{(t, 0)}))\bigg|  \E[\| \nabla L_k(\theta_k^{(t, \tau)}; \zeta_k) - \nabla L_k(\theta_k^{(t, \tau)}) \|^2]\nonumber\\
&\qquad +  \eta^2 \| \frac{1}{K} \sumk \nabla \tf_k(\theta_k^{(t, \tau)}, \btheta_k^{(t, 0)}) \|^2 \\
&\stackrel{(c), (d)}{\leq} \eta^2 (1 + 4\lambda \Gamma)^2 \frac{\sigma^2}{K}  +  \eta^2 \| \frac{1}{K} \sumk \nabla \tf_k(\theta_k^{(t, \tau)}, \btheta_k^{(t, 0)}) \|^2 
\end{align}
Here (a) is because for any random vector $X$, we have $\E[\|X\|^2] = \E[\| X - \E[X]\|^2] + \| \E[X] \|^2 $, (b) is due to the fact that $\E[\| \sumk x_k - \E[x_k] \|^2] = \sumk \E[\| x_k - \E[x_k] \|^2]$ for all $\{x_k\}_{k=1}^K$ which are independent (c) is due to $\forall \theta \in \mathcal{H}: |1 + \frac{\lambda}{K-1} \sumk (r_k(\theta) - r_k(\theta'))| \leq (1 + \lambda \Gamma)$ and (d) is because of the bounded variance of stochastic gradients of $L_k$ : $\E[\| \nabla L_k(\theta; \zeta_k) - \nabla L_k(\theta) \|^2] \leq \sigma^2$.

Thus: 
\begin{align}
&\E[F(\btheta^{(t, \tau + 1)})] 
\\&= \E[F(\btheta^{(t, \tau)})]  - \eta \langle\nabla F(\btheta^{(t, \tau)}), \frac{1}{K}\sumk \nabla \tf_k(\theta_k^{(t, \tau)}, \btheta^{(t, 0)}) \rangle \nonumber \\& \qquad + \frac{(1 + 4\lambda \Gamma)\beta + 8\lambda B^2}{2} \eta^2 \| \frac{1}{K} \sumk \nabla \tf_k(\theta_k^{(t, \tau)}, \btheta_k^{(t, 0)}) \|^2 +  \frac{(1 + 4\lambda \Gamma  )\beta + 8\lambda B^2}{2} \eta^2 (1 + 4 \lambda \Gamma)^2 \frac{\sigma^2}{K} \\
 &\stackrel{(e)}{=} \E[F(\btheta^{(t, \tau)})] \nonumber \\& - \frac{\eta}{2} \bigg[ \| \nabla F(\btheta^{(t, \tau)}) \|^2 + \| \frac{1}{K}\sumk \nabla \tf_k(\theta_k^{(t, \tau)}, \btheta^{(t, 0)}) \|^2 - \|\nabla F(\btheta^{(t, \tau)}) - \frac{1}{K}\sumk \nabla \tf_k(\theta_k^{(t, \tau)}, \btheta^{(t, 0)}) \|^2\bigg] \nonumber \\&+ \frac{(1 + 4\lambda \Gamma  )\beta + 8\lambda B^2}{2} \eta^2 \| \frac{1}{K} \sumk \nabla \tf_k(\theta_k^{(t, \tau)}, \btheta_k^{(t, 0)}) \|^2 \nonumber \\& +  \frac{(1 + 4\lambda \Gamma )\beta + 8\lambda B^2}{2} \eta^2 (1 +4 \lambda \Gamma)^2 \frac{\sigma^2}{K} \\
 &=  \E[F(\btheta^{(t, \tau)})]  - \frac{\eta}{2} \| \nabla  F(\btheta^{(t, \tau)})  \|^2 +  \frac{\eta}{2} \|\nabla F(\btheta^{(t, \tau)}) - \frac{1}{K}\sumk \nabla \tf_k(\theta_k^{(t, \tau)}, \btheta^{(t, 0)}) \|^2 \nonumber \\&\qquad -  \frac{\eta - \eta^2((1 + 4\lambda \Gamma )\beta + 8\lambda B^2)}{2}  \| \frac{1}{K}\sumk \nabla \tf_k(\theta_k^{(t, \tau)}, \btheta^{(t, 0)}) \|^2 \nonumber \\&\qquad +  \frac{(1 + 4\lambda \Gamma  )\beta + 8\lambda B^2}{2} \eta^2 (1 + 4\lambda \Gamma)^2 \frac{\sigma^2}{K} 
 \end{align}
 \begin{align}
&\leq \E[F(\btheta^{(t, \tau)})]  - \frac{\eta}{2} \| \nabla  F(\btheta^{(t, \tau)})  \|^2 +  \frac{\eta}{2} \|\nabla F(\btheta^{(t, \tau)}) - \frac{1}{K}\sumk \nabla \tf_k(\theta_k^{(t, \tau)}, \btheta^{(t, 0)}) \|^2\nonumber\\ & \qquad+\frac{(1 + 4\lambda \Gamma )\beta + 8\lambda B^2}{2} \eta^2 (1 + 4\lambda \Gamma)^2 \frac{\sigma^2}{K} \tag{This is for the choice of $\eta \leq \frac{1}{(1+4\lambda \Gamma)L + 8 \lambda B^2}$}\\
&\stackrel{\eqref{line:bound-gradient-correction}}{\leq} \E[F(\btheta^{(t,\tau)})] - \frac{\eta}{2} \| \nabla F(\btheta^{(t,\tau)}) \|^2 +4 \eta^3 I^2 (1 + 4 \lambda)^2 B^2 \bigg[(1 + 4 \lambda \Gamma)^2 \beta^2  + 32 \lambda^2 B^4  \bigg]\nonumber\\& \qquad+\frac{(1 + 4\lambda \Gamma  )\beta + 8\lambda B^2}{2} \eta^2 (1 + 4 \lambda \Gamma)^2 \frac{\sigma^2}{K}.
\end{align}
(e) is using $\|a - b\|^2 = \|a\|^2 + \|b\|^2 - 2 \langle a,b \rangle$. 

Rearranging and dividing both sides by $\frac{\eta}{2}$ and summing over $t \in \{1, \dots, T\}$ synchronization rounds and $\tau \in \{1,\dots, I\}$ local steps, and finally dividing by $T I$:
\begin{align}
\E[\| \nabla F(\btheta^{(t,\tau)}) \|^2] &\leq \frac{2}{\eta} [ \E[F(\btheta^{(t, \tau)})] - \E[F(\btheta^{(t,\tau+1)})] ] \nonumber\\& \qquad + 2 \eta^2 I^2 (1 + 4 \lambda)^2 B^2 \bigg[(1 + 4 \lambda \Gamma)^2 \beta^2  + 32 \lambda^2 B^4  \bigg] \nonumber\\& \qquad+((1 + 4\lambda \Gamma  )\beta + 8\lambda B^2)(1 + 4 \lambda \Gamma)^2 \eta  \frac{\sigma^2}{K} \\
\frac{1}{T I} \sum_{t = 1}^T \sum_{\tau = 1}^I \E[\| \nabla F(\btheta^{(t, \tau)}) \|^2] &\leq 2\frac{1}{\eta T I} \E[(F(\btheta^{(1, 1)})] -  \E[F(\btheta^{(T, I)})]) \nonumber\\& \qquad+  2 \eta^2 I^2 (1 + 4 \lambda)^2 B^2 \bigg[(1 + 4 \lambda \Gamma)^2 \beta^2  + 32 \lambda^2 B^4  \bigg] \nonumber\\&  \qquad +((1 + 4\lambda \Gamma  )\beta + 8\lambda B^2)(1 + 4 \lambda \Gamma)^2 \eta  \frac{\sigma^2}{K} \\
&\stackrel{F^{*} := \arg\min_\theta F(\theta)}{\leq} 2\frac{1}{\eta T I} (F(\btheta^{(1,1)}) -  F^*) \nonumber\\& \qquad+ 2 \eta^2 I^2 (1 + 4 \lambda)^2 B^2 \bigg[(1 + 4 \lambda \Gamma)^2 \beta^2  + 32 \lambda^2 B^4  \bigg]\nonumber\\& \qquad + ((1 + 4\lambda \Gamma  )\beta + 8\lambda B^2)(1 + 4\lambda \Gamma)^2 \eta  \frac{\sigma^2}{K}.\qedhere
\end{align} 
\end{proof}

\setcounter{theorem}{\value{savedtheoremtwo}}
\addtocounter{theorem}{-1}
\begin{theorem}[Convergence to a solution of federated learning]
\label{theorem:conv_L_appendix}
Let $\theta_k^{(t, \tau)}$ refer to the model of client $k$ after $t$ communication rounds and $\tau$ local steps and let $\btheta^{(t, \tau)} := \frac{1}{K} \sumk \theta_k^{(t, \tau)}$. Let $T$ be the total number of communication rounds and $I$ be the number of local steps between each communication round. Under Assumption~\ref{assumptions}, for $\eta \leq \frac{1}{(1+4\lambda \Gamma)\beta + 8 \lambda B^2}$, the sequence of models $\{\btheta^{(t, \tau)}\}_{t \geq 0, \tau \geq 0}$ generated by Algorithm~\ref{alg:algoname} satisfies:
\begin{align}
\frac{1}{T I} \sum_{t = 1}^T \sum_{\tau = 1}^I &\E[\| \nabla L(\btheta^{(t,\tau)}) \|^2] \leq 2\frac{1}{\eta T I} (L(\btheta^{(1,1)}) -  L^*) +  \underbrace{32 \lambda^2 \Gamma^2 B^2}_{\substack{\text{neighborhood of the solution} \\ \text{relative to the heterogeneity }}} + \eta \frac{\sigma^2}{K} \xi_2\\& +  8 \beta^2  \eta^2 I^2 (1 + 4\lambda \Gamma)^2 B^2 \nonumber,
\end{align} 
with $\xi_2 := ((1 + 4\lambda \Gamma  )\beta + 8\lambda B^2)(1 + 4\lambda \Gamma)^2$.

If we further assume that $L(\theta)$ is $\mu$-Polyak-Łojasiewicz (PL), that is $\forall \theta \in \mathcal{H}: \frac12\| \nabla L(\theta) \| \geq \mu (L(\theta) - L^{*})$, then we have that:
\begin{align}
&\frac{1}{T I} \sum_{t = 1}^T \sum_{\tau = 1}^I \E[L(\btheta^{(t,\tau)}) - L^*] \leq 4\frac{\mu}{\eta T I} (L(\btheta^{(1,1)}) -  L^*) \nonumber\\&  + 2 \eta \frac{\sigma^2}{K} \xi_2 +  16 \beta^2  \eta^2 I^2 (1 + 4\lambda \Gamma)^2 B^2+ \underbrace{62 \mu \lambda^2 \Gamma^2 B^2}_{\mbox{Loss in utility}}.
\end{align}

\end{theorem}

The main difference from the proof of Theorem~\ref{theorem:conv_F_appendix} is the bound on the difference between the applied gradient  $\nabla \tf_k(\theta_k^{(t, \tau)}, \btheta^{(t, 0)})$ and the gradient of the original federated learning \mbox{problem $\nabla L_k(\btheta^{(t, \tau)})$}:
\begin{align}
    &\|\nabla \tf_k(\theta_k^{(t, \tau)}, \btheta^{(t, 0)}) - \nabla L_k(\btheta^{(t, \tau)}) \|^2  \nonumber\\&= \bigg\| \bigg(1 + 4 \frac{\lambda}{K - 1 } \sumkkp [r_k(\btheta^{(t, 0)}) - r_{k'}(\btheta^{(t, 0) })] \bigg) \nabla L_k(\theta_k^{(t, \tau)}) - \nabla L_k(\btheta^{(t, \tau)})  \bigg\|^2\\ &\leq 2 \underbrace{\| \nabla L_k(\theta_k^{(t, \tau)}) - \nabla L_k(\btheta^{(t, \tau)})\|^2}_{\mbox{Controllable using smoothness}} \nonumber \\&+ 2 \bigg\| \bigg(4 \frac{\lambda}{K - 1} \sumkkp [r_k(\btheta^{(t, 0)}) - r_{k'}(\btheta^{(t, 0) })] \bigg) \nabla L_k(\theta_k^{(t, \tau)})  \bigg\|^2\\
    &\stackrel{\|\nabla L_k(\theta)\| \leq B}{\leq} 2 \beta^2 \| \theta_k^{(t, \tau)} - \btheta^{(t, \tau)} \|^2+ \frac{32 \lambda^2}{(K - 1)^2}\bigg| \bigg(\sumkkp [r_k(\btheta^{(t, 0)}) - r_{k'}(\btheta^{(t, 0) })] \bigg)   \bigg|^2 B^2\\
    &\stackrel{Lemma~\ref{eq:diff_local_average}}{\leq} 8 \beta^2  \eta^2 I^2 (1 + 4\lambda \Gamma)^2 B^2 + \frac{32 \lambda^2}{(K - 1)^2}\bigg| \bigg(\sumkkp [r_k(\btheta^{(t, 0)}) - r_{k'}(\btheta^{(t, 0) })] \bigg)   \bigg|^2 B^2 \\
    &\leq 8 \beta^2  \eta^2 I^2 (1 + 4\lambda \Gamma)^2 B^2 + \frac{32 \lambda^2}{K - 1} \sumkkp \underbrace{\bigg| r_k(\btheta^{(t, 0)}) - r_{k'}(\btheta^{(t, 0) })    \bigg|^2}_{\leq \Gamma^2} B^2 \\
    &\leq 32 \lambda^2 \Gamma^2 B^2+ 8 \beta^2  \eta^2 I^2 (1 + 4\lambda \Gamma)^2 B^2.
\end{align}
The results in Theorem~\ref{theorem:conv_L_appendix} follows after applying the smoothness of $L_k$ between two consecutive virtual models $\btheta^{(t, \tau)}$ and $\btheta^{(t, \tau + 1)}$, similarly as in the proof of Theorem~\ref{theorem:conv_F_appendix}.

\section{ADDITIONAL EXPERIMENTAL DETAILS AND RESULTS}

\subsection{Approximation of $L_k^*$}
\label{subsec:l_k_s}
In the implementation of \algoname{}, we start with an initialization phase in which clients run local training with the global model architecture and hyperparameters that are shared by the server. In practice, the client already has an approximation of $L_k^*$ from previous local training before joining federated learning. Each client runs multiple epochs of SGD until one of the stopping criteria is met, the maximum number of allowed epochs is reached, or the training loss history curve flattens. This heuristic convergence test is not guaranteed to reach the global optimum, especially in nonconvex cases. For the experiments on linear models, we use a logistic regression model with nonseparable data, resulting in a strongly convex local objective and thus a single global optimum; therefore, our convergence check works in this case. We leave more elaborate strategies to approximate $L_k^*$ for future work.

\subsection{Dirichlet Splitting}\label{appendix:EMNIST}

Dirichlet splitting consists of generating a probability vector with the size of the number of clients using the Dirichlet distribution with a given parameter $\alpha$. The parameter $\alpha$ controls the level of heterogeneity or uniformity in the generated probability vector, with $\alpha \to 0$ resulting in very heterogeneous probability vectors and $\alpha \to \infty$ in very uniform ones. To generate a data split across $10$ clients, we generate probability vectors for each target value of $\{0,1,\dots,61\}$ with the size of the number of clients for a given $\alpha$ parameter. Following the the generated probability vectors we partition the data of each target value across clients.  No data points are repeated in multiple clients and the probability vectors are generated independently. 

\subsection{Computational Resources} All experiments have been run by simulating federated learning on one machine (no communication between nodes). The machine has a NVIDIA A10 GPU chip with a Linux operating system. The results can be reproduced with the seeds [0, 42, 100, 200] for model initialization and the splitting of the data across clients. All experiments have been run for 10000 epochs with full batch gradients to ensure convergence of all methods. 

\subsection{Models and Hyperparameters} We learn a linear model and a convolutional layer neural network that classifies an image $x \in \mathbb{R}^{784}$ as a single label $y \in [62]$. We use SGD for local steps with learning rates tuned for each algorithm based on preliminary experiments. For q-FFL, we use a grid search on $q$ in $\{0.001,0.01,0.1,1,2,5,10\}$ as proposed in the original work \citep{lifair}. We run \algoname{} with multiple values of $\lambda$ in $\{ 0.1, 0.3, 0.5 0.7, 1.0, 2.0, 3.0, 5.0 \}$ to illustrate the utility-fairness trade-off.


\subsection{Additional Results on Synthetic Data}

\begin{figure*}
\centering 
\includegraphics[scale = 0.4]{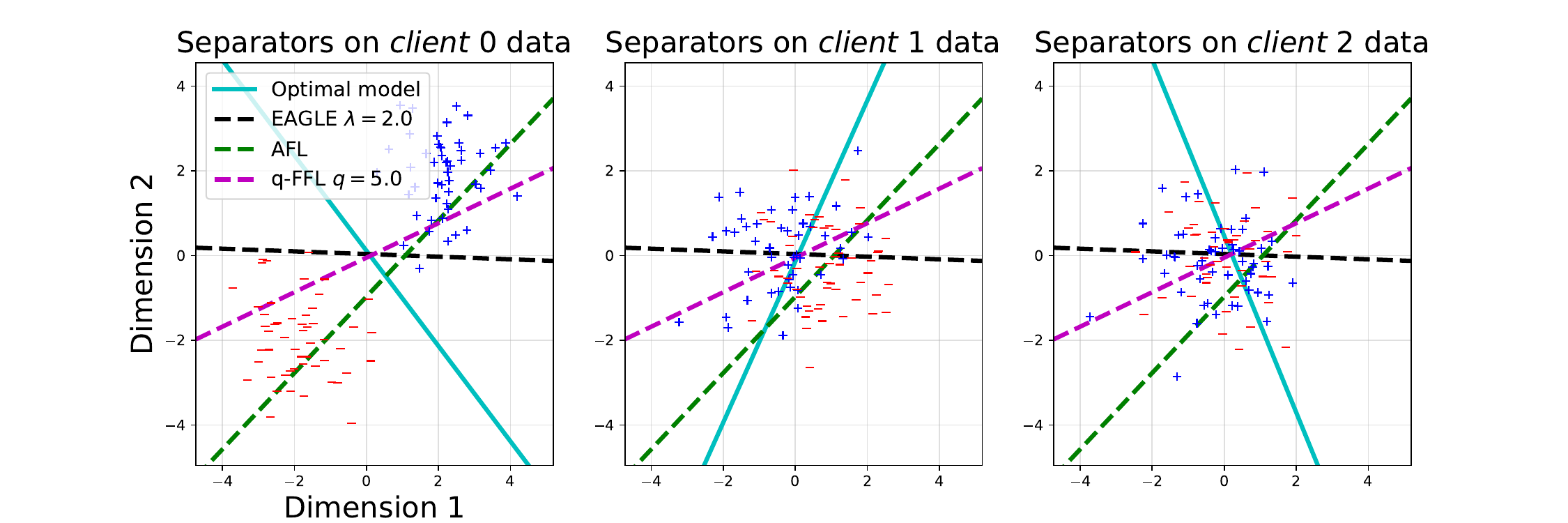} 
\caption{Synthetic data distributions of the three clients. Rotating data resulted in $\clt~1$ and $\clt~2$ sharing the same optimal separator which is different than the one for $\clt~0$. This difference resuls in conflicting gradient directions between local optimizers for different clients.} 
\label{fig:rotated-data} 
\end{figure*}

Figure~\ref{fig:synthetic-losses} shows the optimization behavior of each algorithm, illustrating that clients with lower optimal local losses can be disadvantaged when enforcing \textit{loss parity}.\begin{figure}
    \centering
    \includegraphics[scale = 0.4]{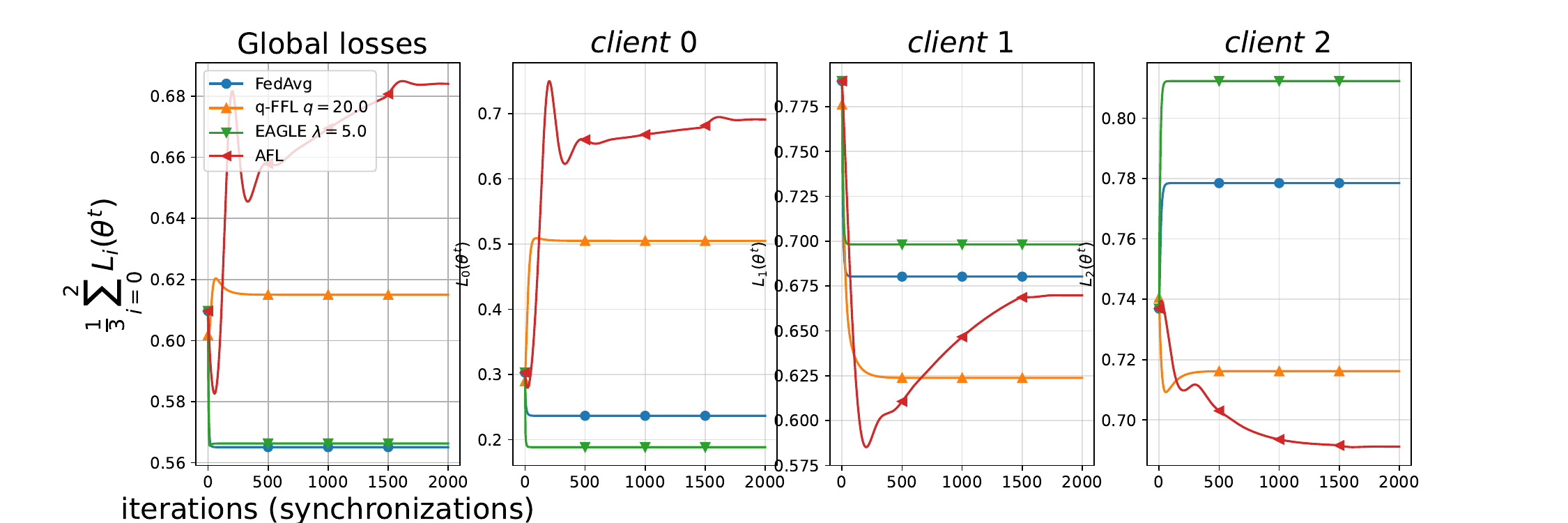} 
    \caption{Training behavior on the synthetic dataset. \textit{AFL} and \textit{q-FFL} increase the training loss of $\clt~0$ to match that of $\clt~2$, but fail to balance the losses of $\clt~1$ and $\clt~2$, which share the same optimal model despite differing in data separability.}
    \label{fig:synthetic-losses}
\end{figure}

\subsection{Additional Results on EMNIST}
\label{app:more_emnist}

This section presents results on enforcing loss gaps parity while preserving utility for \algoname{} and the impact of the hyperparameter $\lambda$. Tables \ref{table:appendix_linear_dir_01}, \ref{table:appendix_linear_dir_05}, \ref{table:appendix_linear_dir_1} and \ref{table:appendix_linear_dir_2000} report the results for the linear model, while Tables \ref{table:appendix_cnn_dir_01}, \ref{table:appendix_cnn_dir_05}, \ref{table:appendix_cnn_dir_1} and \ref{table:appendix_cnn_dir_2000}  display the outcomes for the CNN model. Across both architectures, we observe that the results remain consistent with the main analysis regardless of the heterogeneity level introduced, confirming the generalizability of our conclusions in different heterogeneity levels.

\begin{table}

\caption{Maximum, minimum and variance of loss gaps and accuracy different baselines with a linear model reported on the test split for a heterogeneous split $dir(\alpha = 0.1)$ with EMNIST dataset, results are aggregated over four independent repetitions.}
\label{table:appendix_linear_dir_01}

\begin{tabular}{l|p{0.17\textwidth}|p{0.17\textwidth}|p{0.17\textwidth}|p{0.17\textwidth}}
\hline
Algorithm & \multicolumn{4}{c}{$\alpha = 0.1$} \\
  & $\max\limits_{k \in [K]} r_k(\theta) \downarrow$ & $\min\limits_{k \in [K]} r_k(\theta) \downarrow$ & accuracy $\uparrow$ & $\mathbb{V}_{k \in [K]} r_k(\theta) \downarrow$   \\
\hline
AFL & 0.152 {\tiny($\pm$ 0.207)} & \textbf{-0.477 {\tiny($\pm$ 0.155)}} & 0.687 {\tiny($\pm$ 0.003)} & 0.037 {\tiny($\pm$ 0.026)} \\
FedAvg & 0.106 {\tiny($\pm$ 0.170)} & -0.459 {\tiny($\pm$ 0.211)} & \textbf{0.692 {\tiny($\pm$ 0.004)}} & 0.033 {\tiny($\pm$ 0.021)} \\
q-FFL $q=1.0$ & 0.104 {\tiny($\pm$ 0.178)} & -0.464 {\tiny($\pm$ 0.193)} & 0.692 {\tiny($\pm$ 0.005)} & 0.032 {\tiny($\pm$ 0.021)} \\
q-FFL $q=3.0$ & 0.109 {\tiny($\pm$ 0.186)} & -0.471 {\tiny($\pm$ 0.178)} & 0.691 {\tiny($\pm$ 0.004)} & 0.032 {\tiny($\pm$ 0.022)} \\
q-FFL $q=5.0$ & 0.118 {\tiny($\pm$ 0.189)} & -0.471 {\tiny($\pm$ 0.169)} & 0.689 {\tiny($\pm$ 0.004)} & 0.033 {\tiny($\pm$ 0.023)} \\
\algoname{} $\lambda=0.1$ & 0.096 {\tiny($\pm$ 0.155)} & -0.450 {\tiny($\pm$ 0.209)} & 0.691 {\tiny($\pm$ 0.007)} & 0.030 {\tiny($\pm$ 0.019)} \\
\algoname{} $\lambda=1.0$ & \textbf{0.073 {\tiny($\pm$ 0.123)}} & -0.381 {\tiny($\pm$ 0.199)} & 0.687 {\tiny($\pm$ 0.003)} & 0.020 {\tiny($\pm$ 0.013)} \\
\algoname{} $\lambda=2.0$ & 0.094 {\tiny($\pm$ 0.166)} & -0.355 {\tiny($\pm$ 0.268)} & 0.681 {\tiny($\pm$ 0.007)} & \textbf{0.019 {\tiny($\pm$ 0.010)}} \\
\algoname{} $\lambda=3.0$ & 0.107 {\tiny($\pm$ 0.171)} & -0.335 {\tiny($\pm$ 0.275)} & 0.677 {\tiny($\pm$ 0.007)} & 0.019 {\tiny($\pm$ 0.010)} \\
\algoname{} $\lambda=5.0$ & 0.203 {\tiny($\pm$ 0.181)} & -0.245 {\tiny($\pm$ 0.284)} & 0.650 {\tiny($\pm$ 0.007)} & 0.020 {\tiny($\pm$ 0.011)} \\
\hline
\end{tabular}

\end{table}

\begin{table}
\caption{Maximum, minimum and variance of loss gaps and accuracy different baselines with a linear model reported on the test split for a less heterogeneous split $dir(\alpha = 0.5)$ with EMNIST dataset, results are aggregated over four independent repetitions.}
\label{table:appendix_linear_dir_05}

\begin{tabular}{l|p{0.17\textwidth}|p{0.17\textwidth}|p{0.17\textwidth}|p{0.17\textwidth}}
\hline
Algorithm & \multicolumn{4}{c}{$\alpha = 0.5$} \\
  & $\max\limits_{k \in [K]} r_k(\theta) \downarrow$ & $\min\limits_{k \in [K]} r_k(\theta) \downarrow$ & accuracy $\uparrow$ & $\mathbb{V}_{k \in [K]} r_k(\theta) \downarrow$  \\
\hline
AFL & -0.949 {\tiny($\pm$ 0.106)} & -1.710 {\tiny($\pm$ 0.054)} & 0.686 {\tiny($\pm$ 0.004)} & 0.053 {\tiny($\pm$ 0.012)} \\
FedAvg & -0.961 {\tiny($\pm$ 0.094)} & -1.715 {\tiny($\pm$ 0.044)} & \textbf{0.691 {\tiny($\pm$ 0.004)}} & 0.053 {\tiny($\pm$ 0.014)} \\
q-FFL $q=1$ & -0.963 {\tiny($\pm$ 0.096)} & -1.717 {\tiny($\pm$ 0.048)} & 0.690 {\tiny($\pm$ 0.004)} & 0.053 {\tiny($\pm$ 0.013)} \\
q-FFL $q=3$ & -0.965 {\tiny($\pm$ 0.097)} & -1.720 {\tiny($\pm$ 0.051)} & 0.690 {\tiny($\pm$ 0.004)} & 0.053 {\tiny($\pm$ 0.013)} \\
q-FFL $q=5.0$ & -0.965 {\tiny($\pm$ 0.098)} & \textbf{-1.723 {\tiny($\pm$ 0.051)}} & 0.690 {\tiny($\pm$ 0.003)} & 0.053 {\tiny($\pm$ 0.012)} \\
\algoname{} $\lambda=0.1$ & -0.989 {\tiny($\pm$ 0.086)} & -1.719 {\tiny($\pm$ 0.045)} & \textbf{0.691} {\tiny($\pm$ 0.002)} & 0.050 {\tiny($\pm$ 0.013)} \\
\algoname{} $\lambda=1.0$ & -1.045 {\tiny($\pm$ 0.066)} & -1.617 {\tiny($\pm$ 0.046)} & 0.679 {\tiny($\pm$ 0.005)} & 0.032 {\tiny($\pm$ 0.011)} \\
\algoname{} $\lambda=2.0$ & \textbf{-1.046 {\tiny($\pm$ 0.062)}} & -1.491 {\tiny($\pm$ 0.037)} & 0.662 {\tiny($\pm$ 0.009)} & 0.022 {\tiny($\pm$ 0.008)} \\
\algoname{} $\lambda=3.0$ & -0.990 {\tiny($\pm$ 0.085)} & -1.410 {\tiny($\pm$ 0.028)} & 0.650 {\tiny($\pm$ 0.013)} & 0.018 {\tiny($\pm$ 0.006)} \\
\algoname{} $\lambda=5.0$ & -0.926 {\tiny($\pm$ 0.124)} & -1.347 {\tiny($\pm$ 0.038)} & 0.640 {\tiny($\pm$ 0.016)} & \textbf{0.016 {\tiny($\pm$ 0.006)}} \\
\hline
\end{tabular}

\end{table}

\begin{table}

\caption{Maximum, minimum and variance of loss gaps and accuracy different baselines with a linear model reported on the test split for a moderately heterogeneous split $dir(\alpha = 1.0)$ with EMNIST dataset, results are aggregated over four independent repetitions.}
\label{table:appendix_linear_dir_1}

\begin{tabular}{l|p{0.17\textwidth}|p{0.17\textwidth}|p{0.17\textwidth}|p{0.17\textwidth}}
\hline
Algorithm & \multicolumn{4}{c}{$\alpha = 1.0$} \\
 & $\max\limits_{k \in [K]} r_k(\theta) \downarrow$ & $\min\limits_{k \in [K]} r_k(\theta) \downarrow$ & accuracy $\uparrow$ & $\mathbb{V}_{k \in [K]} r_k(\theta) \downarrow$  \\
\hline
AFL & -1.495 {\tiny($\pm$ 0.141)} & -2.011 {\tiny($\pm$ 0.067)} & 0.688 {\tiny($\pm$ 0.002)} & 0.028 {\tiny($\pm$ 0.016)} \\
FedAvg & -1.503 {\tiny($\pm$ 0.132)} & -2.015 {\tiny($\pm$ 0.074)} & 0.693 {\tiny($\pm$ 0.002)} & 0.028 {\tiny($\pm$ 0.018)} \\
q-FFL $q=1$ & -1.505 {\tiny($\pm$ 0.134)} & -2.015 {\tiny($\pm$ 0.072)} & 0.693 {\tiny($\pm$ 0.002)} & 0.028 {\tiny($\pm$ 0.017)} \\
q-FFL $q=3$ & -1.507 {\tiny($\pm$ 0.135)} & -2.016 {\tiny($\pm$ 0.071)} & 0.692 {\tiny($\pm$ 0.002)} & 0.028 {\tiny($\pm$ 0.017)} \\
q-FFL $q=5.0$ & -1.508 {\tiny($\pm$ 0.136)} & -2.017 {\tiny($\pm$ 0.070)} & 0.692 {\tiny($\pm$ 0.002)} & 0.028 {\tiny($\pm$ 0.017)} \\
\algoname{} $\lambda=0.1$ & \textbf{-1.513 {\tiny($\pm$ 0.131)}} & \textbf{-2.019 {\tiny($\pm$ 0.069)}} & \textbf{0.693 {\tiny($\pm$ 0.001)}} & 0.028 {\tiny($\pm$ 0.017)} \\
\algoname{} $\lambda=1.0$ & -1.497 {\tiny($\pm$ 0.150)} & -1.972 {\tiny($\pm$ 0.034)} & 0.685 {\tiny($\pm$ 0.003)} & 0.023 {\tiny($\pm$ 0.014)} \\
\algoname{} $\lambda=2.0$ & -1.457 {\tiny($\pm$ 0.163)} & -1.898 {\tiny($\pm$ 0.035)} & 0.670 {\tiny($\pm$ 0.008)} & 0.019 {\tiny($\pm$ 0.011)} \\
\algoname{} $\lambda=3.0$ & -1.407 {\tiny($\pm$ 0.160)} & -1.849 {\tiny($\pm$ 0.047)} & 0.657 {\tiny($\pm$ 0.012)} & 0.017 {\tiny($\pm$ 0.009)} \\
\algoname{} $\lambda=5.0$ & -1.333 {\tiny($\pm$ 0.155)} & -1.780 {\tiny($\pm$ 0.069)} & 0.639 {\tiny($\pm$ 0.015)} & \textbf{0.015 {\tiny($\pm$ 0.008)}} \\
\hline
\end{tabular}

\end{table}

\begin{table}

\caption{Maximum, minimum and variance of loss gaps and accuracy different baselines with a linear model reported on the test split for homogeneous split $dir(\alpha = 2000.0)$ (IID data) with EMNIST dataset, results are aggregated over four independent repetitions.}
\label{table:appendix_linear_dir_2000}

\begin{tabular}{l|p{0.17\textwidth}|p{0.17\textwidth}|p{0.17\textwidth}|p{0.17\textwidth}}
\hline
Algorithm & \multicolumn{4}{c}{$\alpha = 2000.0$} \\
 & $\max\limits_{k \in [K]} r_k(\theta) \downarrow$ & $\min\limits_{k \in [K]} r_k(\theta) \downarrow$ & accuracy $\uparrow$ & $\mathbb{V}_{k \in [K]} r_k(\theta) \downarrow$  \\
\hline
AFL & -2.075 {\tiny($\pm$ 0.126)} & -2.571 {\tiny($\pm$ 0.087)} & 0.692 {\tiny($\pm$ 0.001)} & 0.017 {\tiny($\pm$ 0.010)} \\
FedAvg & -2.076 {\tiny($\pm$ 0.126)} & -2.570 {\tiny($\pm$ 0.087)} & \textbf{0.693 {\tiny($\pm$ 0.002)}} & 0.017 {\tiny($\pm$ 0.010)} \\
q-FFL $q=1.0$ & -2.072 {\tiny($\pm$ 0.126)} & -2.565 {\tiny($\pm$ 0.087)} & 0.692 {\tiny($\pm$ 0.002)} & \textbf{0.017 {\tiny($\pm$ 0.010)}} \\
q-FFL $q=3.0$ & -2.072 {\tiny($\pm$ 0.126)} & -2.566 {\tiny($\pm$ 0.086)} & 0.692 {\tiny($\pm$ 0.002)} & 0.017 {\tiny($\pm$ 0.010)} \\
q-FFL $q=5.0$ & -2.072 {\tiny($\pm$ 0.126)} & -2.566 {\tiny($\pm$ 0.086)} & 0.692 {\tiny($\pm$ 0.001)} & 0.017 {\tiny($\pm$ 0.010)} \\
EAGLE $\lambda=0.1$ & \textbf{-2.080 {\tiny($\pm$ 0.127)}} & \textbf{-2.581 {\tiny($\pm$ 0.088)}} & 0.692 {\tiny($\pm$ 0.001)} & 0.018 {\tiny($\pm$ 0.010)} \\
EAGLE $\lambda=1.0$ & -2.079 {\tiny($\pm$ 0.127)} & -2.580 {\tiny($\pm$ 0.088)} & 0.692 {\tiny($\pm$ 0.000)} & 0.018 {\tiny($\pm$ 0.010)} \\
EAGLE $\lambda=2.0$ & -2.059 {\tiny($\pm$ 0.132)} & -2.554 {\tiny($\pm$ 0.092)} & 0.684 {\tiny($\pm$ 0.002)} & 0.018 {\tiny($\pm$ 0.010)} \\
EAGLE $\lambda=3.0$ & -2.029 {\tiny($\pm$ 0.137)} & -2.521 {\tiny($\pm$ 0.092)} & 0.673 {\tiny($\pm$ 0.005)} & 0.017 {\tiny($\pm$ 0.010)} \\
EAGLE $\lambda=5.0$ & -2.068 {\tiny($\pm$ 0.130)} & -2.572 {\tiny($\pm$ 0.085)} & 0.690 {\tiny($\pm$ 0.001)} & 0.018 {\tiny($\pm$ 0.011)} \\
\hline
\end{tabular}

\end{table}

\begin{table}

\caption{Maximum, minimum and variance of loss gaps and accuracy different baselines with a CNN model reported on the test split for a heterogeneous split $dir(\alpha = 0.1)$ with EMNIST dataset, results are aggregated over four independent repetitions.}
\label{table:appendix_cnn_dir_01}

\begin{tabular}{l|p{0.17\textwidth}|p{0.17\textwidth}|p{0.17\textwidth}|p{0.17\textwidth}}
\hline
Algorithm & \multicolumn{4}{c}{$\alpha = 0.1$} \\
  & $\max\limits_{k \in [K]} r_k(\theta) \downarrow$ & $\min\limits_{k \in [K]} r_k(\theta) \downarrow$ & accuracy $\uparrow$ & $\mathbb{V}_{k \in [K]} r_k(\theta) \downarrow$  \\
\hline
AFL & 0.144 {\tiny($\pm$ 0.096)} & -0.497 {\tiny($\pm$ 0.159)} & 0.843 {\tiny($\pm$ 0.001)} & 0.036 {\tiny($\pm$ 0.018)} \\
FedAvg & 0.294 {\tiny($\pm$ 0.196)} & \textbf{-0.512 {\tiny($\pm$ 0.131)}} & 0.844 {\tiny($\pm$ 0.005)} & 0.047 {\tiny($\pm$ 0.020)} \\
q-FFL $q=1.0$ & 0.097 {\tiny($\pm$ 0.067)} & -0.473 {\tiny($\pm$ 0.120)} & 0.835 {\tiny($\pm$ 0.003)} & 0.028 {\tiny($\pm$ 0.013)} \\
q-FFL $q=3.0$ & 0.244 {\tiny($\pm$ 0.089)} & -0.330 {\tiny($\pm$ 0.147)} & 0.800 {\tiny($\pm$ 0.009)} & 0.030 {\tiny($\pm$ 0.015)} \\
q-FFL $q=5.0$ & 0.386 {\tiny($\pm$ 0.089)} & -0.201 {\tiny($\pm$ 0.159)} & 0.765 {\tiny($\pm$ 0.011)} & 0.032 {\tiny($\pm$ 0.017)} \\
\algoname{} $\lambda=0.1$ & 0.091 {\tiny($\pm$ 0.136)} & -0.506 {\tiny($\pm$ 0.101)} & \textbf{0.848 {\tiny($\pm$ 0.003)}} & 0.029 {\tiny($\pm$ 0.017)} \\
\algoname{} $\lambda=1.0$ & \textbf{0.063 {\tiny($\pm$ 0.058)}} & -0.348 {\tiny($\pm$ 0.103)} & 0.844 {\tiny($\pm$ 0.009)} & \textbf{0.013 {\tiny($\pm$ 0.003)}} \\
\algoname{} $\lambda=2.0$ & 0.178 {\tiny($\pm$ 0.118)} & -0.315 {\tiny($\pm$ 0.084)} & 0.838 {\tiny($\pm$ 0.008)} & 0.017 {\tiny($\pm$ 0.010)} \\
\algoname{} $\lambda=3.0$ & 0.153 {\tiny($\pm$ 0.101)} & -0.294 {\tiny($\pm$ 0.077)} & 0.834 {\tiny($\pm$ 0.007)} & 0.014 {\tiny($\pm$ 0.007)} \\
\algoname{} $\lambda=5.0$ & 0.192 {\tiny($\pm$ 0.137)} & -0.268 {\tiny($\pm$ 0.083)} & 0.828 {\tiny($\pm$ 0.007)} & 0.017 {\tiny($\pm$ 0.009)} \\
\hline
\end{tabular}

\end{table}

\begin{table}
\caption{Maximum, minimum and variance of loss gaps and accuracy different baselines with a CNN model reported on the test split for a less heterogeneous split $dir(\alpha = 0.5)$ with EMNIST dataset, results are aggregated over four independent repetitions.}
\label{table:appendix_cnn_dir_05}

\begin{tabular}{l|p{0.17\textwidth}|p{0.17\textwidth}|p{0.17\textwidth}|p{0.17\textwidth}}
\hline
Algorithm & \multicolumn{4}{c}{$\alpha = 0.5$} \\
  & $\max\limits_{k \in [K]} r_k(\theta) \downarrow$ & $\min\limits_{k \in [K]} r_k(\theta) \downarrow$ & accuracy $\uparrow$ & $\mathbb{V}_{k \in [K]} r_k(\theta) \downarrow$  \\
\hline
AFL & -0.416 {\tiny($\pm$ 0.132)} & -1.480 {\tiny($\pm$ 0.674)} & 0.841 {\tiny($\pm$ 0.004)} & 0.110 {\tiny($\pm$ 0.137)} \\
FedAvg & -0.409 {\tiny($\pm$ 0.130)} & -1.477 {\tiny($\pm$ 0.675)} & 0.840 {\tiny($\pm$ 0.003)} & 0.112 {\tiny($\pm$ 0.140)} \\
q-FFL $q=1.0$ & -0.534 {\tiny($\pm$ 0.111)} & -1.558 {\tiny($\pm$ 0.670)} & 0.843 {\tiny($\pm$ 0.002)} & 0.109 {\tiny($\pm$ 0.141)} \\
q-FFL $q=3.0$ & -0.476 {\tiny($\pm$ 0.115)} & -1.524 {\tiny($\pm$ 0.681)} & 0.833 {\tiny($\pm$ 0.002)} & 0.114 {\tiny($\pm$ 0.145)} \\
q-FFL $q=5.0$ & -0.416 {\tiny($\pm$ 0.108)} & -1.470 {\tiny($\pm$ 0.679)} & 0.820 {\tiny($\pm$ 0.005)} & 0.115 {\tiny($\pm$ 0.144)} \\
\algoname{} $\lambda=0.1$ & \textbf{-0.546 {\tiny($\pm$ 0.103)}} & \textbf{-1.567 {\tiny($\pm$ 0.671)}} & \textbf{0.850 {\tiny($\pm$ 0.002)}} & 0.108 {\tiny($\pm$ 0.140)} \\
\algoname{} $\lambda=1.0$ & -0.502 {\tiny($\pm$ 0.080)} & -1.202 {\tiny($\pm$ 0.289)} & 0.832 {\tiny($\pm$ 0.011)} & 0.044 {\tiny($\pm$ 0.037)} \\
\algoname{} $\lambda=2.0$ & -0.483 {\tiny($\pm$ 0.069)} & -1.078 {\tiny($\pm$ 0.257)} & 0.823 {\tiny($\pm$ 0.004)} & 0.032 {\tiny($\pm$ 0.029)} \\
\algoname{} $\lambda=3.0$ & -0.466 {\tiny($\pm$ 0.068)} & -1.011 {\tiny($\pm$ 0.247)} & 0.816 {\tiny($\pm$ 0.005)} & 0.028 {\tiny($\pm$ 0.026)} \\
\algoname{} $\lambda=5.0$ & -0.424 {\tiny($\pm$ 0.071)} & -0.938 {\tiny($\pm$ 0.223)} & 0.805 {\tiny($\pm$ 0.011)} & \textbf{0.024 {\tiny($\pm$ 0.022)}} \\
\hline
\end{tabular}

\end{table}

\begin{table}
\caption{Maximum, minimum and variance of loss gaps and accuracy different baselines with CNN model reported on the test split for a moderately heterogeneous split $dir(\alpha = 1.0)$ with EMNIST dataset, results are aggregated over four independent repetitions}
\label{table:appendix_cnn_dir_1}

\begin{tabular}{l|p{0.17\textwidth}|p{0.17\textwidth}|p{0.17\textwidth}|p{0.17\textwidth}}
\hline
Algorithm & \multicolumn{4}{c}{$\alpha = 1$} \\
 & $\max\limits_{k \in [K]} r_k(\theta) \downarrow$ & $\min\limits_{k \in [K]} r_k(\theta) \downarrow$ & accuracy $\uparrow$ & $\mathbb{V}_{k \in [K]} r_k(\theta) \downarrow$  \\
\hline
AFL & -0.726 {\tiny($\pm$ 0.078)} & -1.367 {\tiny($\pm$ 0.082)} & 0.845 {\tiny($\pm$ 0.003)} & 0.038 {\tiny($\pm$ 0.012)} \\
FedAvg & -0.742 {\tiny($\pm$ 0.052)} & -1.400 {\tiny($\pm$ 0.065)} & 0.843 {\tiny($\pm$ 0.003)} & 0.039 {\tiny($\pm$ 0.014)} \\
q-FFL $q=1.0$ & -0.760 {\tiny($\pm$ 0.077)} & -1.431 {\tiny($\pm$ 0.096)} & 0.846 {\tiny($\pm$ 0.004)} & 0.041 {\tiny($\pm$ 0.016)} \\
q-FFL $q=3.0$ & -0.764 {\tiny($\pm$ 0.079)} & -1.437 {\tiny($\pm$ 0.100)} & 0.842 {\tiny($\pm$ 0.003)} & 0.042 {\tiny($\pm$ 0.018)} \\
q-FFL $q=5.0$ & -0.745 {\tiny($\pm$ 0.081)} & -1.432 {\tiny($\pm$ 0.102)} & 0.839 {\tiny($\pm$ 0.004)} & 0.043 {\tiny($\pm$ 0.018)} \\
EAGLE $\lambda=0.1$ & \textbf{-0.794 {\tiny($\pm$ 0.073)}} & \textbf{-1.461 {\tiny($\pm$ 0.097)}} & \textbf{0.851 {\tiny($\pm$ 0.001)}} & 0.041 {\tiny($\pm$ 0.017)} \\
EAGLE $\lambda=1.0$ & -0.744 {\tiny($\pm$ 0.081)} & -1.294 {\tiny($\pm$ 0.063)} & 0.833 {\tiny($\pm$ 0.012)} & 0.029 {\tiny($\pm$ 0.010)} \\
EAGLE $\lambda=2.0$ & -0.621 {\tiny($\pm$ 0.115)} & -1.062 {\tiny($\pm$ 0.101)} & 0.809 {\tiny($\pm$ 0.014)} & 0.019 {\tiny($\pm$ 0.007)} \\
EAGLE $\lambda=3.0$ & -0.567 {\tiny($\pm$ 0.136)} & -0.939 {\tiny($\pm$ 0.126)} & 0.800 {\tiny($\pm$ 0.016)} & 0.016 {\tiny($\pm$ 0.005)} \\
EAGLE $\lambda=5.0$ & -0.543 {\tiny($\pm$ 0.151)} & -0.864 {\tiny($\pm$ 0.117)} & 0.789 {\tiny($\pm$ 0.017)} & \textbf{0.012 {\tiny($\pm$ 0.003)}} \\
\hline
\end{tabular}

\end{table}

\begin{table}
\caption{Maximum, minimum and variance of loss gaps and accuracy different baselines with a CNN model reported on the test split for homogeneous split $dir(\alpha = 2000.0)$ (IID data) with EMNIST dataset, results are aggregated over four independent repetitions.}
\label{table:appendix_cnn_dir_2000}

\begin{tabular}{l|p{0.17\textwidth}|p{0.17\textwidth}|p{0.17\textwidth}|p{0.17\textwidth}}
\hline
Algorithm & \multicolumn{4}{c}{$\alpha = 2000$} \\
 & $\max\limits_{k \in [K]} r_k(\theta) \downarrow$ & $\min\limits_{k \in [K]} r_k(\theta) \downarrow$ & accuracy $\uparrow$ & $\mathbb{V}_{k \in [K]} r_k(\theta) \downarrow$  \\
\hline
AFL & -1.140 {\tiny($\pm$ 0.063)} & -1.591 {\tiny($\pm$ 0.060)} & 0.845 {\tiny($\pm$ 0.002)} & 0.015 {\tiny($\pm$ 0.002)} \\
FedAvg & -1.140 {\tiny($\pm$ 0.063)} & -1.591 {\tiny($\pm$ 0.060)} & 0.845 {\tiny($\pm$ 0.002)} & 0.015 {\tiny($\pm$ 0.002)} \\
q-FFL $q=1.0$ & -1.170 {\tiny($\pm$ 0.053)} & -1.615 {\tiny($\pm$ 0.053)} & 0.849 {\tiny($\pm$ 0.002)} & 0.015 {\tiny($\pm$ 0.001)} \\
q-FFL $q=3.0$ & -1.180 {\tiny($\pm$ 0.062)} & -1.605 {\tiny($\pm$ 0.069)} & 0.845 {\tiny($\pm$ 0.004)} & 0.014 {\tiny($\pm$ 0.002)} \\
q-FFL $q=5.0$ & -1.193 {\tiny($\pm$ 0.040)} & -1.618 {\tiny($\pm$ 0.054)} & 0.846 {\tiny($\pm$ 0.003)} & 0.014 {\tiny($\pm$ 0.002)} \\
EAGLE $\lambda=0.1$ & \textbf{-1.217 {\tiny($\pm$ 0.051)}} & \textbf{-1.650 {\tiny($\pm$ 0.060)}} & \textbf{0.850 {\tiny($\pm$ 0.000)}} & 0.014 {\tiny($\pm$ 0.002)} \\
EAGLE $\lambda=1.0$ & -1.209 {\tiny($\pm$ 0.050)} & -1.637 {\tiny($\pm$ 0.060)} & 0.848 {\tiny($\pm$ 0.001)} & 0.014 {\tiny($\pm$ 0.002)} \\
EAGLE $\lambda=2.0$ & -1.162 {\tiny($\pm$ 0.051)} & -1.582 {\tiny($\pm$ 0.062)} & 0.838 {\tiny($\pm$ 0.006)} & 0.013 {\tiny($\pm$ 0.002)} \\
EAGLE $\lambda=3.0$ & -0.962 {\tiny($\pm$ 0.164)} & -1.364 {\tiny($\pm$ 0.173)} & 0.801 {\tiny($\pm$ 0.026)} & \textbf{0.013 {\tiny($\pm$ 0.002)}} \\
EAGLE $\lambda=5.0$ & -0.570 {\tiny($\pm$ 0.074)} & -0.976 {\tiny($\pm$ 0.078)} & 0.690 {\tiny($\pm$ 0.022)} & 0.013 {\tiny($\pm$ 0.003)} \\
\hline
\end{tabular}

\end{table}

\end{document}